%% file: preprint.tex
\documentclass{article} % For LaTeX2e
\usepackage[margin=1.2in]{geometry}
\input{math_commands.tex}

\usepackage{hyperref}
\usepackage{url}

\usepackage{amsmath}
\usepackage{amssymb}
\usepackage{amsthm}
\usepackage{graphicx}
\usepackage{adjustbox} 
\usepackage{subcaption}
\usepackage{etoolbox}
\usepackage[utf8]{inputenc} % allow utf-8 input
\usepackage[T1]{fontenc}    % use 8-bit T1 fonts
\usepackage{hyperref}       % hyperlinks
\usepackage{url}            % simple URL typesetting
\usepackage{booktabs}       % professional-quality tables
\usepackage{amsfonts}       % blackboard math symbols
\usepackage{nicefrac}       % compact symbols for 1/2, etc.
\usepackage{microtype}      % microtypography
\usepackage{xcolor}  % colors
\usepackage{authblk}
\usepackage{graphicx}
\DeclareRobustCommand{\lowstar}[1][0.9ex]{\raisebox{#1}{\scriptsize\textasteriskcentered}}

\newrobustcmd\B{\DeclareFontSeriesDefault[rm]{bf}{b}\bfseries}

\newtheorem{proposition}{Proposition}
\newtheorem{definition}{Definition}

\newcommand{\bb}[1]{\mathbb{#1}}

\newcommand{\EE}{\bb{E}}

\title{On Uniformly Scaling Flows: A Density-Aligned Approach to Deep One-Class Classification}

\author[ ]{Faried {Abu Zaid}\thanks{Equal contribution.}}
\author[1]{Tim Katzke\lowstar}
\author[1]{Emmanuel Müller}
\author[1]{Daniel Neider}
\affil[1]{Research Center Trustworthy Data Science and Security, TU Dortmund University}

\affil[ ]{\texttt{fariedaz@gmail.com}, \texttt{tim.katzke@tu-dortmund.de}}

\begin{document}
\maketitle

\begin{abstract}
Unsupervised anomaly detection is often framed around two widely studied paradigms. Deep one-class classification, exemplified by Deep SVDD, learns compact latent representations of normality, while density estimators realized by normalizing flows directly model the likelihood of nominal data. In this work, we show that uniformly scaling flows (USFs), normalizing flows with a constant Jacobian determinant, precisely connect these approaches. Specifically, we prove how training a USF via maximum‑likelihood reduces to a Deep SVDD objective with a unique regularization that inherently prevents representational collapse. This theoretical bridge implies that USFs inherit both the density faithfulness of flows and the distance-based reasoning of one-class methods. We further demonstrate that USFs induce a tighter alignment between negative log‑likelihood and latent norm than either Deep SVDD or non‑USFs, and how recent hybrid approaches combining one-class objectives with VAEs can be naturally extended to USFs. Consequently, we advocate using USFs as a drop‑in replacement for non‑USFs in modern anomaly detection architectures. Empirically, this substitution yields consistent performance gains and substantially improved training stability across multiple benchmarks and model backbones for both image‑level and pixel‑level detection. These results unify two major anomaly detection paradigms, advancing both theoretical understanding and practical performance.

\end{abstract}

\section{Introduction}
Unsupervised anomaly detection seeks to identify rare deviations from normality without access to labeled outliers and underpins safety‑critical applications in medical diagnostics, cybersecurity, and industrial inspection~\cite{fernando2021deep,LiuXWLWZJ24}. For such complex underlying data, performance and reliability hinge on how well “normal” structure is captured in a latent space, where typical samples form a compact, semantically meaningful region, so that unlikely patterns can be identified.
Estimating the data’s true underlying density provides a principled foundation to accomplish this task~\cite{RuffKVMSKDM21}.

Within the landscape of leading deep unsupervised anomaly detection methods, there exist two seemingly complementary lines: one‑class classification --- epitomized by Deep SVDD~\cite{RuffGDSVBMK18} --- which forgoes explicit density mapping and learns encoders that concentrate normal data into a minimal‑volume support (e.g., a hypersphere), and distribution maps, which estimate a distribution over embeddings~\cite{LiuXWLWZJ24}. Among the latter, a predominant example are normalizing flows, generative density estimator models which learn bijective transformations between the complex data space and simple base distributions and enable exact likelihoods~\cite{papamakarios2021normalizing}. Figure~\ref{fig:schematic} schematically contrasts hyperspherical one‑class mapping against a normalizing flow with isotopic gaussian base distribution. While one‑class methods offer geometric simplicity and efficiency, the risk trivial solutions, hypersphere collapse, and limited density fidelity~\cite{ChongRKB20,KluettermannKM25,ZhangD21}. Flows provide explicit likelihoods but their scores can be confounded by input‑dependent log‑determinant terms and higher computational cost~\cite{ liao_jacobian_2021,papamakarios2021normalizing}. Still, despite different mechanisms, both lines aim to delineate a high‑mass region of normality, while a rigorous account of when and how exactly their objectives coincide and how to leverage this in practice has been missing.

We close this gap by introducing a theoretically grounded bridge via uniformly scaling flows (USFs) --- normalizing flows with a constant Jacobian determinant --- by showing that USFs are a special case of a Deep SVDDs. In particular, we formally establish that
USFs trained via standard maximum likelihood estimation (MLE) are mathematically equivalent to Deep SVDD objectives with a distinctive regularization scheme. Since this regularization directly
relates to the determinant of the Jacobian matrix, it penalizes volume collapse in a way that provably circumvents the well-known degeneracy issues of standard deep one-class classification. Motivated by these properties, we advocate USFs as drop‑in replacements for state-of-the-art, non-USF‑based anomaly detectors. Eliminating input‑dependent log‑determinant terms makes scores depend only on position in the base latent, yielding more predictable level‑set geometry and empirically showing consistent gains across standard industrial AD benchmarks in both image‑ and pixel‑level detection with improved performance stability. 

% Structure of Paper
The remainder of this paper is structured as follows. \S2 reviews the necessary background on anomaly detection, Deep SVDD, and normalizing flows, formally defining uniformly scaling flows. \S3 situates our work within the broader landscape of related research on deep one-class classification and flow-based anomaly detection. Our core theoretical contribution is presented in \S4, where we prove the equivalence between maximum likelihood training of a USF and a regularized Deep SVDD objective, formally bridging the two paradigms. Additionally, we analyze the advantages of USFs with respect to common Deep SVDD and non-USF failure modes theoretically and empirically. Motivated by these theoretical advantages, \S5 presents an extensive empirical study where we replace standard flows with USFs in modern architectures (FastFlow, CFlow, U-Flow), showing consistent performance gains and significantly improved stability across the MVTec AD and VisA benchmarks. Finally, \S6 concludes.

\begin{figure}[tp!]
  \centering
  \includegraphics[width=\textwidth]{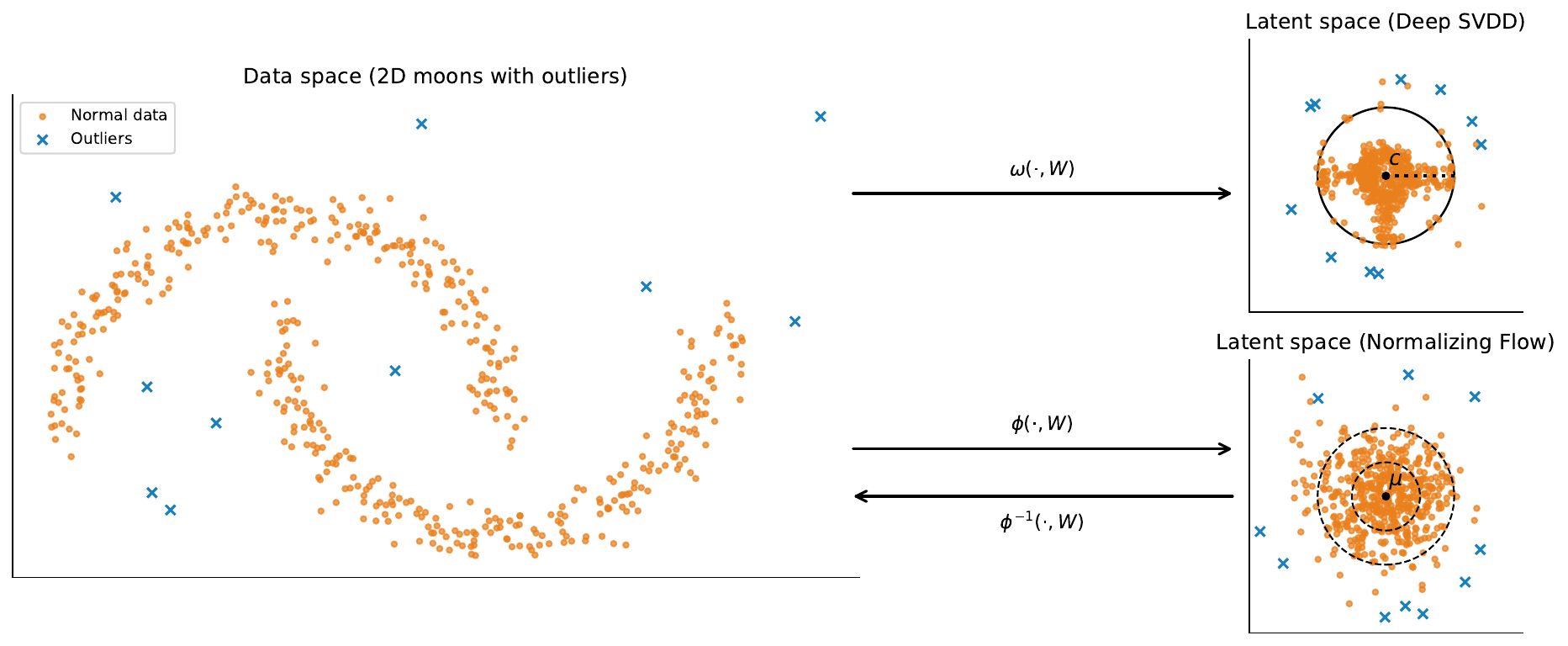}
  \caption{Schematic comparison of Deep SVDD and Normalizing Flows. The left panel shows the input data space: a nonlinearly structured 2D “moons” dataset with outliers. Top-right depicts Deep SVDD’s latent space: a learned mapping $\omega(\cdot, W)$ concentrates normal samples near a center $c$ inside a hypersphere, while anomalies lie outside; the non-uniform distribution emphasizes representation bias. Bottom-right shows the Normalizing Flows latent space: an invertible mapping $\phi(\cdot, W)$ moves data towards a standard Gaussian (illustrated by 1$\sigma$ and 2$\sigma$ contours) and relegates outliers to low-density regions, with $\phi^{-1}(\cdot, W)$ indicating the reverse mapping.}
  \label{fig:schematic}
\end{figure}

\section{Preliminaries}\label{sec:preliminaries}

To first define unsupervised anomaly detection formally, let $X$ be a random variable governed by a probability density function $p_X(x)$ that models the true distribution of normal data. An observation $x$ is deemed anomalous if its likelihood under this distribution falls below a predefined threshold $\tau$, i.e.,
\begin{align}
  \label{eq:anomaly_detection}
    \{\,x \mid p_X(x) < \tau\}.
\end{align}

This framing characterizes anomalies as those points residing in the low–probability tails of the normal data distribution, consistent with the definition that "an anomaly is an observation that deviates considerably from some concept of normality"~\cite{RuffKVMSKDM21}. One can transform this likelihood criterion into an anomaly score by considering the Shannon information, $-\ln p_X(x)$ (basis $e$, nats).

\paragraph{Deep SVDD}

Rather than estimating the true density \(p_X(x)\) and applying a threshold test as in Equation~\ref{eq:anomaly_detection}, Deep Support Vector Data Description (Deep SVDD)~\cite{RuffGDSVBMK18} casts anomaly detection as a one‐class classification problem: it learns a mapping s.t. the latent representations of normal samples are confined within a minimal‐radius hypersphere around a fixed center point.

To this end, Deep SVDD parameterizes a deep neural network $\omega(x,w)\colon X\to\mathbb{R}^d$ with layer weights $w=(w^1,\dots,w^L)$ and learns these weights by optimizing the one-class loss
\begin{align}\label{eqn:deepsvdd_loss}
\min_{w}\;\frac{1}{n}\sum_{i=1}^n \lVert \omega(x_i,w) - c\rVert^2
\;+\;\frac{\lambda}{2}\sum_{l=1}^L\lVert w^l\rVert^2_F
\end{align}
where $\{x_i\}_{i=1}^n$ are the normal training samples and $c\in\mathbb{R}^d$ is the hypersphere center. The second term corresponds to a weight decay regularizer with hyperparameter $\lambda > 0$ where $\lVert \cdot\rVert_F$ denotes the Frobenius norm. To avoid the trivial collapse solution of mapping every data point to the center $c$, usually $c$ is fixed to a vector $\neq 0$, bias terms are omitted and only unbounded activation functions are used. In practice, the network weights $w$ are initialized by pretraining an autoencoder on the normal data and adopting its encoder parameters $w_{\mathrm{init}}$ for $\omega$, after which the center $c$ is set to the mean of $\{\omega(x_i,w_{\mathrm{init}})\}_{i=1}^n$. At inference, the squared distance $\lVert \omega(x,w) - c\rVert^2$ between an input’s latent representation and the hypersphere center is treated as the anomaly measure.

\paragraph{Normalizing Flows}
Flow models are commonly used as density estimators and generative models. 
Formally, a flow $\phi$ transforms a
usually complex data (or target) distribution $X$ into a typically simple base (or latent) distribution $B=\phi(X)$ using a continuous invertible map with continuous inverse,
i.e. a diffeomorphism. The map $\phi_w$ is implemented by an invertible neural network with parameters $w$. We obtain a density estimator and a generative model by applying the flow in both directions:

\begin{enumerate}
  \item Sampling is performed by first sampling $z \sim B$ and then computing
  the inverse map $\phi^{-1}_w(z)$.
  
  \item The likelihood
   $p_w(X=x)$ is computable with the change of variables formula~\cite{folland1999real}: 
  \begin{eqnarray}p_w (X=x) = \left|\det J_{\phi_w}(x) \right| p_B (\phi_w (x)), \label{eqn:var_change}
  \end{eqnarray}
\end{enumerate}
where $J_{\phi_w}(x) = \frac{\partial \phi_w}{\partial x}$ denotes the Jacobi matrix of $\phi_w$. Flows are typically trained to estimate the density of some unknown distribution $X$ by fixing the base distribution and maximizing the log-likelihood of samples from $X$ using Equation \ref{eqn:var_change}. In this case, we write $\hat{p}_X^{w}(x)$ in order to avoid confusion with the true density $p_X$. The by far most common base distribution is a standard Normal $\mathcal{N}(0, I)$, which motivates the name normalizing flow.

While most neural networks architectures are intrinsically differentiable, they typically do not enforce
bijectivity. One needs to design specific architectures
that restrict the hypothesis space to diffeomorphisms. Most discrete (non-ODE) architectures rely on various types of coupling layers \cite{dinh2014nice, Dinh2017, Kingma2018, Midgley2023}. Another crucial building block are bijective affine transforms \cite{hoogeboom2019emerging}, which are often restricted to some subclass like (soft) permutations or householder reflections \cite{dinh2014nice, tomczak2016vae}. 

\paragraph{Uniformly Scaling Flows}

A core contribution of this work is to reveal a profound theoretical connection between the prominent Deep SVDD anomaly detection framework and a well-known subclass of normalizing flows. 
A \emph{uniformly scaling flow (USF)} is a diffeomorphism $\phi: \mathbb{R}^d \to \mathbb{R}^d$ whose Jacobian determinant is constant throughout the input space:
\begin{equation}
\det J_{\phi}(x) = \kappa \quad \forall \mathbf{x} \in \mathbb{R}^d
\end{equation}
for some constant $\kappa \in \mathbb{R}\setminus\{0\}$. This property induces uniform volume scaling across all regions of the latent space. Note that if $\phi=\phi_w$ is parameterized, then $|\det J_{\phi_w}(x)| = \psi_{\text{det}}^\phi(w)$ for some function $\psi_{\text{det}}^\phi$ that depends only on the parameters $w$.

While non-uniformly scaling flow architectures (e.g., RealNVP~\cite{Dinh2017}) have largely superseded uniformly scaling variants like NICE~\cite{dinh2014nice} in anomaly detection architectures, recent theoretical insights demonstrate that USFs retain unique advantages in specialized domains including neuro-symbolic verification~\cite{ZaidNY2024}. 

The USFlows library~\footnote{\url{https://github.com/aai-institute/USFlows}} 
provides modern and robust implementation of USFs via additive coupling layers, adjoint bijective affine group actions parameterized by learnable general LU-decomposed bijective affine transformations, and invertible $1\times1$ convolutions. We will use the USFlows architecture for all our flow experiments in Section~\ref{sec:comparison}. In order to provide fair baseline comparisons, we also implemented a non-US counterpart, NonUSFlows, that is identical to the USFlows architecture, except that additive coupling is replaced by affine coupling, which is non-US \cite{Dinh2017}. 
In the following, we describe the architectures in detail. We use
$m\odot x$ to denote the component wise (Hadamard) product between two tensors of same topology.

\paragraph{USFlows Architecture}
According to \cite{ZaidNY2024}, the USFlow architecture is build from masked additive coupling layers \cite{dinh2014nice, ma_macow_2019} and bijective affine transformations parameterized by LU-decomposed matrices \cite{chan_lunet_2023}. 

A masked additive coupling layer $C: \mathbb{R}^d \to \mathbb{R}^d$ (or any other fixed tensor topology) defines a bijection based on an arbitrary \emph{conditioner} function $f: \mathbb{R}^d \to \mathbb{R}^d$ and a $\{0,1\}$-mask $m$ as $C(x) = x + (1 - m)\odot f(m\odot x)$. By construction, $m\odot x = m\odot C(x)$ and hence $C^{-1}(y) = y - (1-m)\odot f(m\odot y)$. 

An LU layer computes a general bijective affine transform $A(x)= Mx + b$. Bijectivity is guaranteed by imposing constraints on the parametrization linear transform: $M=LU$, where $L$/$U$ are lower/upper triangular matrices with $L_{ii}=1$ and $U_{ii}\neq 0$ for all indices $i$.    

In order to to allow uniform handling of different input topologies, USFlows employs one-star-convolutions (convolutions with kernel size 1 along all spatial dimensions, adapted to input topology) for the bijective affine transforms, where an LU-decomposed channel transform ensures bijectivity. For the conditioner of the coupling transforms, general adaptive convolutions (convolutions with homogeneous kernel size along all spatial dimensions, adapted to input topology) with layer normalization and gating are employed.

A USFlow model is build from blocks of shape
$B_i = A_i^{-1}\circ C_i \circ A_i$,
where $A_i$ is a bijective affine transform and $C$ is an additive coupling layer. 
The complete flow is of shape
$
{\phi = A_{n+1} \circ B_n \circ B_{n-1} \circ \cdots \circ B_1,}
$
where the final affine transform ensures that the flow can have an arbitrary (constant) Jacobian determinant. Note that each block is volume preserving since determinants of 
$A_i$ and $A_i^{-1}$ cancel each other and additive coupling layers are volume preserving by design~\cite{dinh2014nice}.

\paragraph{NonUSFlows Architecture}
The NonUSFlows architecture is identical except that we use masked affine coupling instead of additive coupling for the blocks. 

A masked affine coupling layer $C: \mathbb{R}^d \to \mathbb{R}^d$ (or any other fixed tensor topology) defines a bijection based on an arbitrary additive conditioner function $f: \mathbb{R}^d \to \mathbb{R}^d$, a multiplicative conditioner function $g: \mathbb{R}^d \to \mathbb{R}^d$ with $g(x)_i\neq 0$ for all $x\in\mathbb{R}^d$, $i\in\{1,\ldots,d\}$ and a $\{0,1\}$-mask $m$ as $C(x) = m\odot x +(1-m)\odot g(m\odot x) \odot x + (1 - m)\odot f(m\odot x)$. Again, since $m\odot x = m\odot C(x)$, the inverse is given by $C^{-1}(y) = m\odot y + (1-m)\odot\frac{y - (1-m)\odot f(m\odot y)}{g(m\odot y)}$. Note that an affine coupling layer with $g(x) = 1$ for all $x$ is equivalent to an additive coupling layer. 

Allowing arbitrary non-zero outputs for $g$ can be numerically unstable in practice (locally exploding determinants). Therefore, we adopt the common practice to apply a final scaled $\tanh$ activation with a configurable scaling parameter to our conditioner network $g$. Moreover, $f$ and $g$ are implemented by a single conditioner that computes both outputs, which further unifies the interface of both models. In all our experiments, the parametrization of the conditioner used in USFlows and NonUSFlows differs only in the final layer to produce the required output shapes, respectively.

\section{Related Work}\label{sec:related_work}

\paragraph{Modern Deep One-class Classification}

Deep one‑class approaches such as Deep SVDD are vulnerable to representational collapse (forcing diverse normals into a single center), and latent misalignment (no incentive to preserve input‑space density/manifold structure).
To make collapse suboptimal at the objective level, follow-up methods inject training signals that enforce feature diversity and margin: DROCC synthesizes near‑manifold “virtual anomalies” and learns to separate them from normals~\cite{GoyalRJS020}; DeepSAD adds a pull–push term using few labeled outliers to tighten the normal region~\cite{RuffVGBMMK20}; and ~\cite{ChongRKB20} regularize against constant embeddings via noise injection and a mini‑batch variance penalty. 
In parallel, a complementary line of work preserves latent fidelity by coupling compactness with reconstruction or probabilistic consistency: DASVDD and DSPSVDD combine Deep SVDD with autoencoder reconstruction to keep features informative and curb collapse~\cite{ZhangD21,HojjatiA24}, while Deep SVDD‑VAE introduces a probabilistic latent that regularizes geometry and improves distributional alignment~\cite{ZhouLZZS21}. In contrast to such extensions and auxiliary tasks, we show how training a USF via MLE is equivalent to optimizing a one-class loss with implicit density regularization and how this avoids representational collapse and benefits latents space alignment.

\paragraph{Normalizing Flows for Anomaly Detection}

Experience with pixel-space flows revealed that maximum-likelihood training steers toward generic low-level regularities. In term, training flows on deep feature embeddings resolves this behavior and markedly improves performance~\cite{kirichenko_why_2020a}. 
DifferNet is the first method that fits a flow on features from a pretrained CNN to demonstrate strong image-level detection on industrial AD~\cite{RudolphWR21}, while CS-Flow extends this line with multi-level/contextualized feature modeling~\cite{RudolphWRW22}. CFlow conditions the flow on spatial indices and backbone features to model the distribution of normal patches per location, enabling accurate localization \cite{GudovskiyIK22}; FastFlow implements a lightweight 2D convolutional flow directly on feature maps, preserving topology end-to-end for real-time anomaly segmentation \cite{YuZWLWTW21}; and U-Flow adopts a U-shaped, multi-scale conditioning design that couples local likelihoods with broader context \cite{TailanianPM24}. 
All of these methods employ variable-volume flows with input-dependent Jacobians. This boosts per-block modeling power, but can confound scoring by entangling base-density and volume terms. Notably, uniformly scaling flows may serve as direct drop-in replacements without having to adjust other methodological components.

\paragraph{Intersection of Normalizing Flows and Deep One-Class Classification}

Normalizing flows and deep one-class classification both define “normality” as a high‑probability (or minimum‑volume) region and flag points outside it. Density thresholding with flows and distance cut-offs with one‑class objectives are thus two lenses to identify the boundary of normal data. 
A growing body of research also proposes hybrid approaches. Flow‑based SVDD instantiates a SVDD objective atop a volume preserving flow, avoiding collapse while learning a tight normal region in latent space~\cite{SenderaSMSPT21}.  One‑Flow similarly uses a flow to find a minimal‑volume hypersphere that contains a target fraction of the data, operationalizing one‑class support estimation~\cite{MaziarkaSSSTS22}, while NFAD trains a flow on normal data and samples low‑density regions to synthesize pseudo‑anomalies that refine one‑class classifiers, turning their optimization into a self-supervised binary task~\cite{RyzhikovBUD21}. 
Despite such hybrids and recurring observations of similarity, the literature still lacks a precise theoretical bridge between the two paradigms.

This motivates our subsequent analysis: when do normalizing flows replicate well-principled one-class approaches, and how can this be leveraged to improve performance and robustness in flow-based anomaly detection?

\section{Comparing Deep SVDDs and Uniformly Scaling Flows}\label{sec:comparison}

This section first establishes a theoretical bridge between Deep SVDD and uniformly scaling flows. It then relates their shared failure modes, delineates native differences in latent space alignment, and shows how dimensionality reduction methods established for Deep SVDD carry over to USFs. 

\subsection{Loss Function}\label{sec:loss_function}

Recall that the training loss of a Deep SVDD given in Equation \ref{eqn:deepsvdd_loss} consist of two terms: the square distance from the center $\| \omega(x,w) - c \|^2$ and a regularization term. 
Comparing the Deep SVDD objective to the maximum likelihood training objective of normalizing flows reveals a close relationship. Specifically, when training a flow \( \phi \) whose Jacobian determinant is independent of
$x$ (i.e., a USF) with base distribution \( \mathcal{N}\left(c, \frac{1}{2}I\right) \) and parameters $w$ on a training dataset $D$, the objective becomes:  
\begin{align}
    \min_w \mathbb{E}_{x\sim X}\left[-\ln \hat{p}_X^w(x)\right] 
    &= \min_w \mathbb{E}_{x\sim D}\left[-\ln \left(\mathcal{N}\left(\phi_w(x); c, \frac{1}{2}I\right) \left|\det J_{\phi_w}(x)\right|\right)\right] \nonumber\\
    &= \min_w \mathbb{E}_{x\sim D}\left[-\ln\exp\left( -\frac{(\phi_w(x) - c)^T 2I (\phi_w(x) - c)}{2}\right) - \ln\psi_{\text{det}}^\phi(w)\right] \nonumber\\
    &= \min_w \mathbb{E}_{x\sim D}\left[ (\phi_w(x) - c)^T (\phi_w(x) - c) - \ln\psi_{\text{det}}^\phi(w)\right] \nonumber\\
    &= \min_w \mathbb{E}_{x\sim D}\left[\|\phi_w(x) - c\|^2\right] - \ln\psi_{\text{det}}^\phi(w)\label{eqn:usflow_deepsvdd_loss}
\end{align}
Note that the constant terms derived from $\ln \mathcal{N}\left(\phi_w(x); c, \frac{1}{2}I\right)$ can be omitted because of the $\min$ operator.
Thus, the objectives differ primarily in the "regularization" term applied only to the weights. We will see in Section \ref{sec:mode_explode} and Section \ref{sec:lsa} that this alternate term is critical to avoid hypersphere collapse and to ensure a meaningful latent space alignment. 

\subsection{Hypersphere Collapse and Exploding Determinants}\label{sec:mode_explode}

Deep SVDD's objective admits a degenerate solution: \emph{hypersphere collapse}, where all inputs map to a single point, minimizing the loss trivially. This is mitigated architecturally (e.g., no biases, fixed center) but limits expressivity.
While follow-up methods attempt to circumvent these limitations (see Section~\ref{sec:related_work}), flow-based models prevent this degeneration theoretically. Namely, architectural constraints ensure bijectivity and maximum likelihood training of a flow minimizes the reverse Kullback-Leibler divergence in latent space. 
We derive this fact well known fact for the sake of self-containedness:
\begin{align*}
    \text{KL}\left(P_{\phi(X)} \bigm\| \mathcal{N}\left(c, \tfrac{1}{2}I\right)\right) 
    &= \mathbb{E}_{x \sim P_X} \left[ \ln \frac{p_X(x) \cdot |\det J_{\phi}(x)|^{-1}}{\mathcal{N}\left(\phi(x); c, \frac{1}{2}I\right)} \right] \\
    &= \underbrace{\mathbb{E}_{x \sim P_X} [\ln p_X(x)]}_{\text{const.}} - \mathbb{E}_{x \sim P_X} \left[ \ln \mathcal{N}\left(\phi(x); c, \frac{1}{2}I\right)  + \ln\psi_{\text{det}}^\phi\right] \\
    &= \EE_{x\sim P_X}[-\ln\hat{p}_w(X=x)] + K, 
\end{align*}
where the last expression is obtained by applying the logarithmic version of the change of variables formula to the second term of the previous expression and $K$ is a constant value.
Hence, training a flow model with maximum likelihood drives the transformed data distribution toward the base distribution $P_B$ (here, $\mathcal{N}(c, \tfrac{1}{2}I)$ to align with Deep SVDD). This holds for any base distribution and any flow.

\paragraph{Implicit Determinant Regularization}  
In practice, normalizing flows are susceptible to a related pathology: \emph{exploding determinants}~\cite{behrmann_understanding_2021,lee_universal_2021,liao_jacobian_2021}. Gradient-based optimization can exploit the change-of-variables formula by inflating $|\det J_{\phi}|$, making the high-density region of the base distribution represent only a vanishingly small portion of the data support. We mitigate this effect with a Bayesian approach for the USFlow architecture, placing a log-normal prior on the absolute value of the determinant to regularize its scale during maximum a posteriori estimation. 
We show that it is possible to define a prior on the weights of an USFlow such that the induced prior on the determinant of the transformation is log-normal.
\begin{definition}
The \emph{bilateral log-normal} is defined on $\mathbb{R}$  as 
\[\textrm{BiLogNormal}(x; \mu, \sigma^2) = \frac{1}{2} \textrm{LogNormal}(|x|; \mu, \sigma^2),\]
where $\textrm{LogNormal}(x; \mu, \sigma^2) = \frac{1}{x\sigma \sqrt{2\pi}}\exp{-\frac{(\ln x - \mu)^2}{2\sigma^2}}$ is the usual log-normal distribution defined on $\mathbb{R}_{\geq 0}$. Note that if $X\sim \textrm{BiLogNormal}(\mu, \sigma^2)$ then $|X|\sim \textrm{LogNormal}(\mu, \sigma^2)$ and $\ln |X| \sim \mathcal{N}(\mu, \sigma^2)$.
\end{definition}
In the following, we carry out the derivation for image topologies, i.e. input tensors of shape $\textrm{Channels}\times\textrm{Height}\times\textrm{Width}$. The derivation for other input topologies is completely analogous.
Now recall that the USFlows Architecture is build from additive coupling and affine transforms. Additive coupling is volume preserving by design, which means that we only need to take care of the affine transforms. The affine transforms of a USFlow are given by bijective $1\times 1$ convolutions with LU decomposed bijective channel transforms. Let $C$ be a $1\times 1$ convolution with channel transform $A = LU$ for lower/upper triangular matrices $L$/$U$ with $L_{ii} = 1$ and $U_{ii} \neq 0$ for all $i$. For a fixed image topology $c\times h \times w$, $|\det J_C(x)| = |\det A|^{hw}$. 
Further, since $|\det A| = \prod_{i=1}^c |U_{ii}|$, we impose an independent bilateral log-normal prior $\textrm{BiLogNormal}(0, \sigma_0^2)$ with $\sigma_0 = \sigma/(chw)$ on each $u_i := U_{ii}$. This yields multiplicative regularization:  
\begin{align}
    p\left(\left|\det J_C(x)\right|\right) &= p\left({\prod}_{i=1}^c |u_{i}|^{hw}\right) \nonumber 
    = p\left(\exp\left(hw{\sum}_{i=1}^c \ln |u_{ii}|\right)\right). \label{eqn:log_normal}
\end{align}
Since $p(\ln |u_i|) = \mathcal{N}(0, \sigma_0^2)$ and all $u_i$ are independent:  
\[
{\textstyle hw\sum_{i=1}^c} \ln |u_{i}| \sim \mathcal{N}(0, \sigma^2) \implies |\det J_{C}(x)| \sim \text{LogNormal}(0, \sigma^2).
\]
Variance control via $\sigma_0$ thus regularizes determinant scale. 

Note that it is sufficient to regularize the affine transforms individually to the desired scale because, except for the last affine transform, all affine transforms appear as adjoint actions $C^{-1}\circ F \circ C$ applied to an additive coupling layer $F$ and therefore determinants cancel block-wise. This prior regularizes the overall transformation as well as the inter-block transforms in a uniform way. Empirically, we did not observe a need to regularize the remaining parameters of the network.

% Ensure consistency in notation:
% - P_X: Data distribution
% - \phi_{\#}P_X: Pushforward measure
% - J_{\phi}: Jacobian matrix
% - \mathcal{D}: Training dataset

\subsection{Latent Space Alignment}\label{sec:lsa}

While kernel-based one-class methods are consistent density level-set estimators~\cite{RuffGDSVBMK18}, Deep SVDD foregoes direct density modeling. Its objective learns compact representations but lacks a mechanism to enforce alignment between input-space density and latent-space geometry. This permits pathological solutions, such as \emph{density inversion}, where a less likely input is mapped closer to the center $c$ than a more likely one, violating fundamental anomaly detection principles while achieving minimal loss.

\begin{proposition}
Let $\mathcal{N}$ be a $d$-dimensional standard normal distribution with $d > 2$. There exists a class of functions $F_\alpha: \mathbb{R}^d \to \mathbb{R}^d$, $\alpha\in\mathbb{R}^+$, such that $F_\alpha$ is non-degenerate for all $\alpha > 0$ and $\lim_{\alpha \to \infty} L(N, F_\alpha) = 0$, yet for all $\alpha>0$ and $x, y \in \mathbb{R}^d$ with $x,y \neq 0$,
\[
\| F_\alpha(x) - c \| < \| F_\alpha(y) - c \| \implies \mathcal{N}(x) < \mathcal{N}(y),
\]
where $L$ is the Deep SVDD loss (without regularization) with center $c$.
\end{proposition}
\begin{proof}
For simplicity, we assume $c = 0$ (otherwise, $F_\alpha$ would also need to shift points towards $c$). We define
\[
F_\alpha(x) = \frac{1}{\alpha|x|},
\]
with \( F_\alpha(0) = 0 \). Consequently, as $|x|$ increases (i.e., as $x$ becomes less probable under $\mathcal{N}$), the center distance $\|F_\alpha(x) - c \| = \| F_\alpha(x) \|$ decreases and vice versa.

The overall loss for a function of this class is
\[
L(\mathcal{N}, F_\alpha) = E_{x \sim \mathcal{N}}\left[ \|F_\alpha(x) - c\|^2 \right] = E_{x \sim \mathcal{N}}\left[ \left( \frac{1}{\alpha|x|} \right)^2 \right] = \frac{1}{\alpha^2} E_{x \sim \mathcal{N}}\left[ \frac{1}{|x|^2} \right].
\]

Since $x$ is a vector in $\mathbb{R}^d$ where each component $x_i$ is an i.i.d standard normal random variable, i.e. $x_i \sim \mathcal{N}(0,1)$, it follows that $|x|^2$ is distributed according to a chi-squared distribution with $d$ degrees of freedom, i.e. $|x|^2 \sim \chi^2(d)$. Therefore, for $d > 2$, the expected value simplifies to $ E_{x \sim N}\left[ \frac{1}{|x|^2} \right] = \frac{1}{d-2} $.
Substituting into the loss gives $L(\mathcal{N}, F_\alpha) = \frac{1}{\alpha^2 (d-2)}$, and thus $\lim_{\alpha \to \infty} L(N, F_\alpha) = 0 $. 
\end{proof}

We choose to ignore the regularization term in our example because we assume such degenerate solutions to be approximated by complex neural networks and the influence of the regularization term w.r.t. the parameterization given in the example can be misleading. For the sake of completeness, however, we would still like to mention that it is not hard to adjust the example such that we can meet arbitrarily small losses with the optimal solution under the Deep SVDD loss function, even if the standard regularization is applied to the given parameters.

\paragraph{Uniformly Scaling Flows}
USFs, in contrast, guarantee density-preserving alignment. From Eq. \ref{eqn:var_change}, a constant Jacobian determinant implies $p_X(x) < p_X(y) \Leftrightarrow p_B(\phi(x)) < p_B(\phi(y))$. For an isotropic Gaussian base distribution, this yields the desired property: $p_X(x) < p_X(y) \Leftrightarrow \|\phi(x)\| > \|\phi(y)\|$. This ensures the latent norm is a faithful anomaly score, a property not guaranteed by non-uniformly scaling flows.

\paragraph{Illustrative Example} We demonstrate this practical advantage on an asymmetrical Gaussian mixture toy dataset (designed to incentivize non-uniform volume transfer), comparing the relationship between data density and latent norms for Deep SVDD, USFs, and non-USFs across dimensions $d=2,8,32,128$. 
The parameters of the asymmetric Gaussian mixture distribution are given as follows:
\begin{itemize}
    \item $\mu_1 = \mathbf{1}_d, \mu_2 = -\mathbf{1}_d$
    \item $\Sigma_1 = I, \Sigma_2 = \textrm{diag}(\theta_d)$,
\end{itemize}
where $\mathbf{1}_d$ denotes the all ones vector in $d$ dimensions and $\theta_d\in\mathbb{R}^d$ is chosen such that half of the entries are $5.0$ and half are $1/2$. This asymmetric choice of covariance matrices is deliberate in order to incentivize inhomogeneous volume transfer via $\det\Sigma_1=1$ and $\det\Sigma_2=2.5^{\frac{d}{2}}$. 
For each model architecture and dimension, we conduct a hyperparameter optimization with 10 runs. For the flow architectures, we use (Non)USFlows with 2-10 coupling blocks and 2-3 layers per conditioner. For Deep SVDD, we use MLP encoders with depth 2–6 following a decreasing  schedule based on the input dimension $d$ (i.e., [$d$, $d$], [$d$, $2d$, $d$], ... ,[$d$, $16d$, $8d$, $4d$, $2d$, $d$]), forgoing bias terms and bounded activations as discussed in the preliminaries. We select the best-performing model per setting by validation loss (one-class for DeepSVDD, NLL for flows). Additionally, we optimize non architecture specific parameters such as learning rate and batch size. 

Figure~\ref{fig:2d_gauss_mapping} provides visual intuition for deviations in density faithfulness of the latent mappings for the best-performing models in the two-dimensional setting. The USF exhibits near-ideal density alignment, the non-USF shows noticeable deviations, and Deep SVDD, lacking any incentive for density alignment, produces a highly inhomogeneous latent structure. A visualization of the underlying differences in volume transfer between the USFlows model and the NonUSFlows model is given in Figure~\ref{fig:2d_gauss_mapping_det}. Because the USF’s transformation has a constant Jacobian determinant by design, it induces uniform scaling, as expected. In contrast, for the non-USF, the discrepancies in density alignment are corrected by the variable volume change of the transformation.

To analyze latent-space alignment as dimensionality increases, Figure~\ref{fig:gauss-mapping-all-norm} plots true likelihoods against latent norms for Deep SVDD, NonUSFlow and USFlow. Only the USF consistently maintains a monotonic relationship between data likelihood and latent norm up to 128 dimensions, as reflected by Spearman’s $\rho$ and Kendall’s $\tau$. At 128 dimensions, we also observed numerical instabilities for NonUSFlows; see Section~\ref{sup:nonusf_instabilities} for details. Despite these clear differences in latent alignment between the flow-based models, their estimated likelihoods remain broadly comparable, as shown in Figure~\ref{fig:gauss-mapping-all-logprob}, which plots estimated against true data likelihoods.

\begin{figure}[htbp!]
  \centering
  \captionsetup[subfigure]{labelformat=parens,justification=centering}
  % first row
  \begin{subfigure}[t]{0.275\textwidth}
    \includegraphics[width=\textwidth]{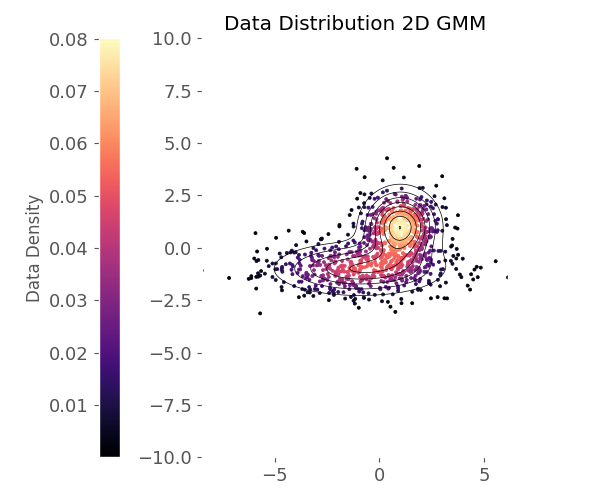}
    \caption{} \label{fig:2d_gmm_contour}
  \end{subfigure}
    \begin{subfigure}[t]{0.235\textwidth}
    \includegraphics[width=\textwidth]{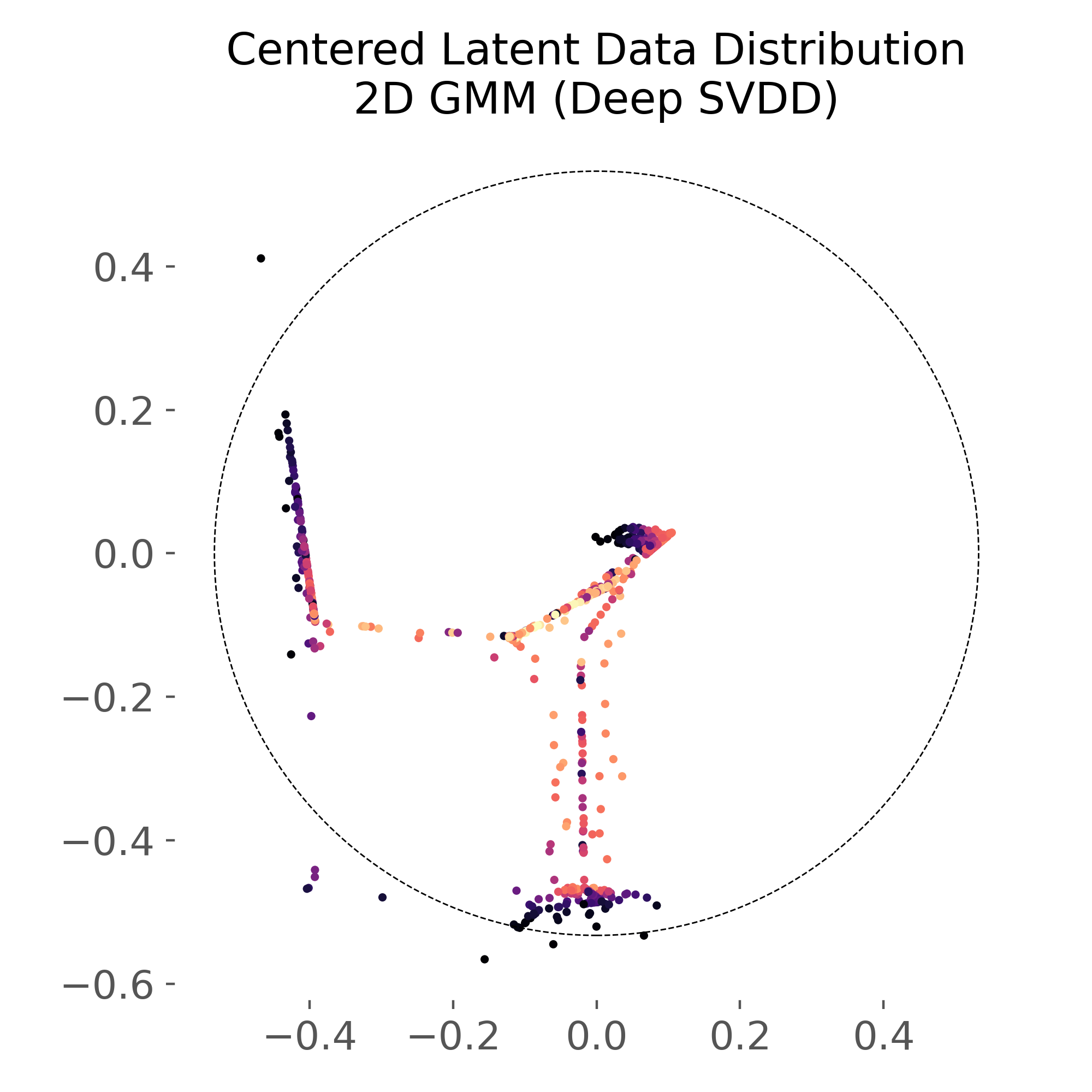}
    \caption{} \label{fig:2d_gmm_latents_deep_svdd}
  \end{subfigure}
  \begin{subfigure}[t]{0.235\textwidth}
    \includegraphics[width=\textwidth]{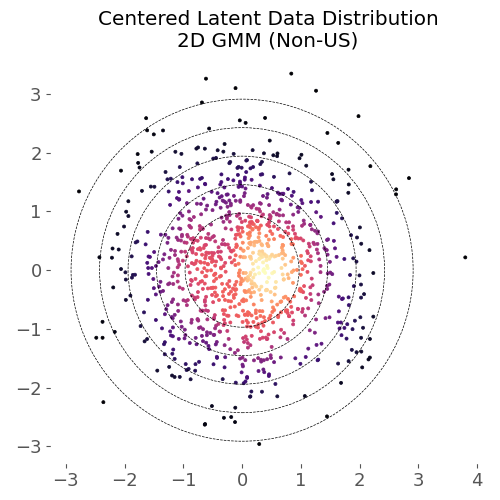}
    \caption{} \label{fig:2d_gmm_latents_nonus}
  \end{subfigure}
  \begin{subfigure}[t]{0.235\textwidth}
    \includegraphics[width=\textwidth]{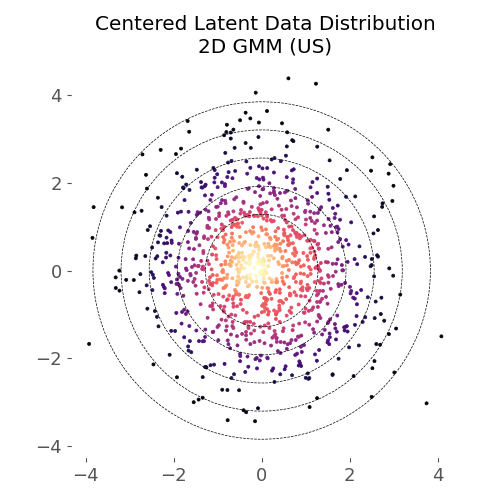}
    \caption{} \label{fig:2d_gmm_latents_us}
  \end{subfigure}\hfill

  \caption{%
  2D Gaussian mixture experiment visualization. (a): True data distribution (b): Latent space of Deep SVDD (c): Latent space of non-USF (d): Latent space of USF. Point color encodes the true data density. Dashed lines show countours of the base distribution ($\sigma$--$3\sigma$) (flows) or the decision threshold (Deep SVDD) based on the 99th distance percentile. The USF shows direct density alignment between data and latent spaces, while Deep SVDD and non-USF show noticable discrepancies.
  }
  \label{fig:2d_gauss_mapping}
\end{figure}

% One-row layout: US (data, latent) | Non-US (data, latent)
\begin{figure*}[htbp!]
  \centering
  \captionsetup[subfigure]{labelformat=parens,justification=centering}

  \begin{subfigure}[t]{0.27\textwidth}
    \includegraphics[width=\linewidth]{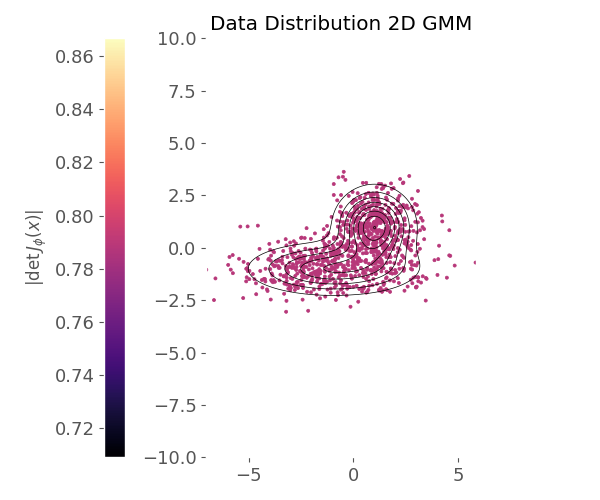}
    \caption{}
    \label{fig:2d_gmm_data_us}
  \end{subfigure}\hfill
  \begin{subfigure}[t]{0.22\textwidth}
    \includegraphics[width=\linewidth]{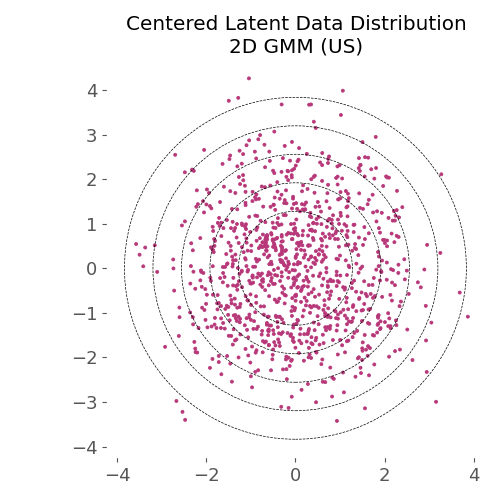}
    \caption{}
    \label{fig:2d_gmm_latent_us}
  \end{subfigure}\hfill\vrule
  \begin{subfigure}[t]{0.27\textwidth}
    \includegraphics[width=\linewidth]{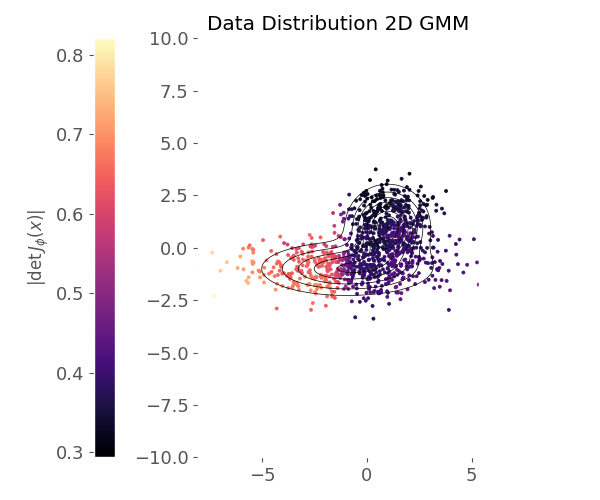}
    \caption{}
    \label{fig:2d_gmm_data_nonus}
  \end{subfigure}\hfill
  \begin{subfigure}[t]{0.22\textwidth}
    \includegraphics[width=\linewidth]{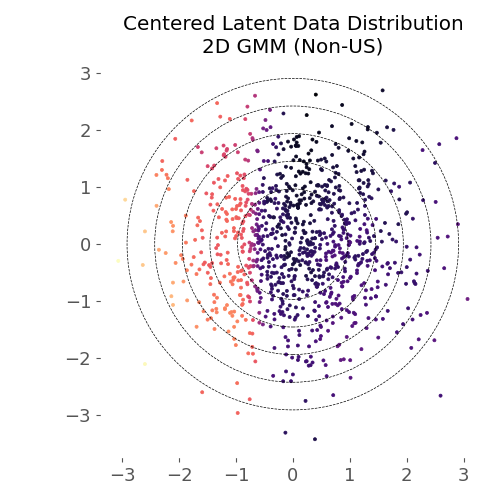}
    \caption{}
    \label{fig:2d_gmm_latent_nonus}
  \end{subfigure}

  \caption{%
  Visualization of the transformation’s volume change in different areas of the data distribution for USFs and non-USFs.
  (a,c): Sample and contours of the true data distribution. The determinant of the respective transformation Jacobian is color coded.
  (b,d): Sample and ideal contours of centered data latents of the (b) USF and the (d) non-USF.
  The dashed contours show the contour lines of the flows base distribution (centered, $\sigma$ -- $3\sigma$). The color of each sample encodes the absolute value of the determinant of the transformation Jacobian.
  The USF enforces a constant Jacobian determinant, ensuring uniform scaling, whereas the non-USF adapts local densities through variable volume changes to correct alignment discrepancies.
  }
  \label{fig:2d_gauss_mapping_det}
\end{figure*}

\begin{figure}[htbp!]
  \centering
  \captionsetup[subfigure]{labelformat=parens,justification=centering}
  % first row
    \begin{subfigure}[t]{0.24\textwidth}
    \includegraphics[width=\textwidth]{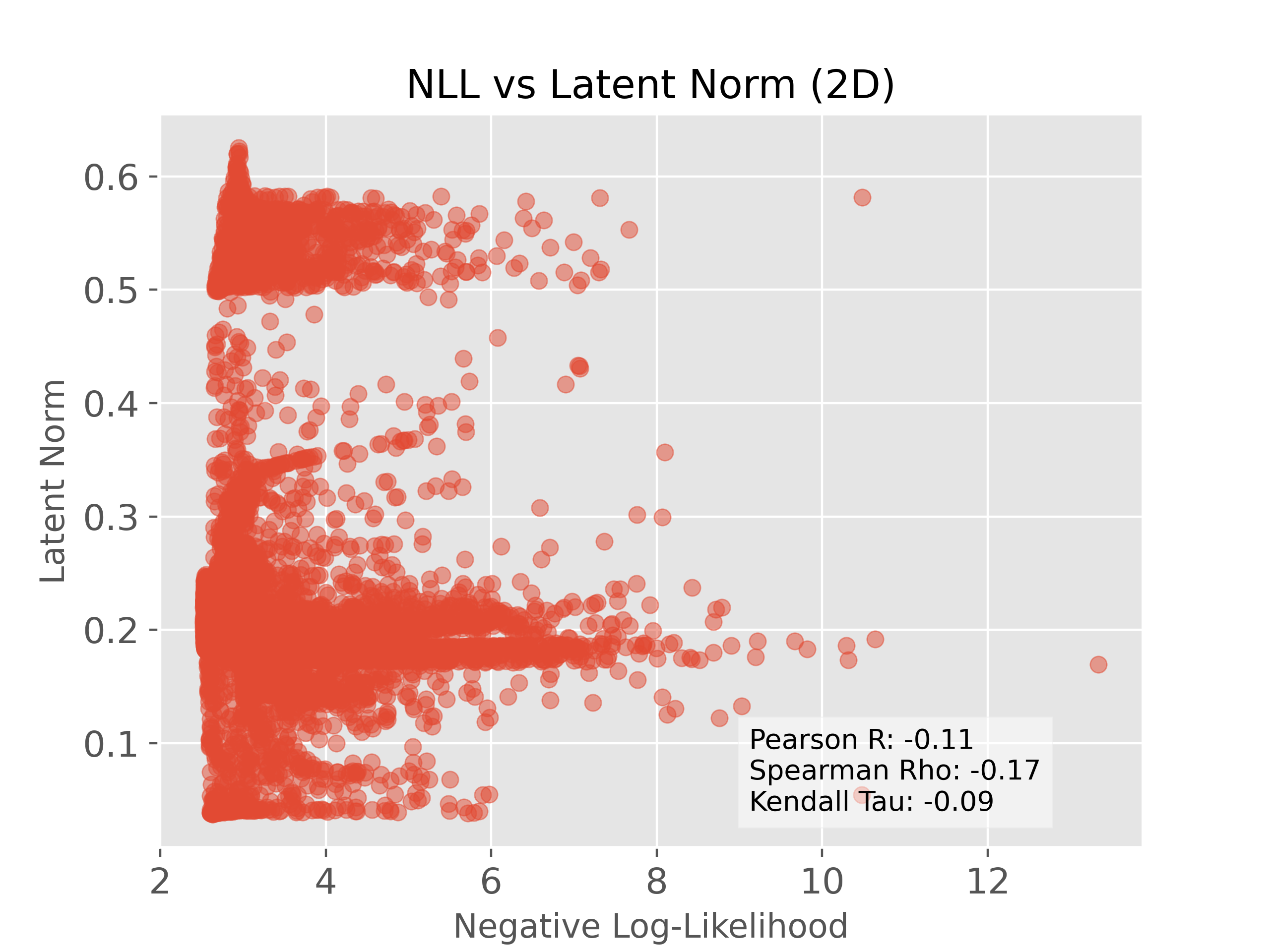}
    \caption{} \label{fig:gmm_nll_vs_norm_2d_deepsvdd}
  \end{subfigure}\hfill
  \begin{subfigure}[t]{0.24\textwidth}
    \includegraphics[width=\textwidth]{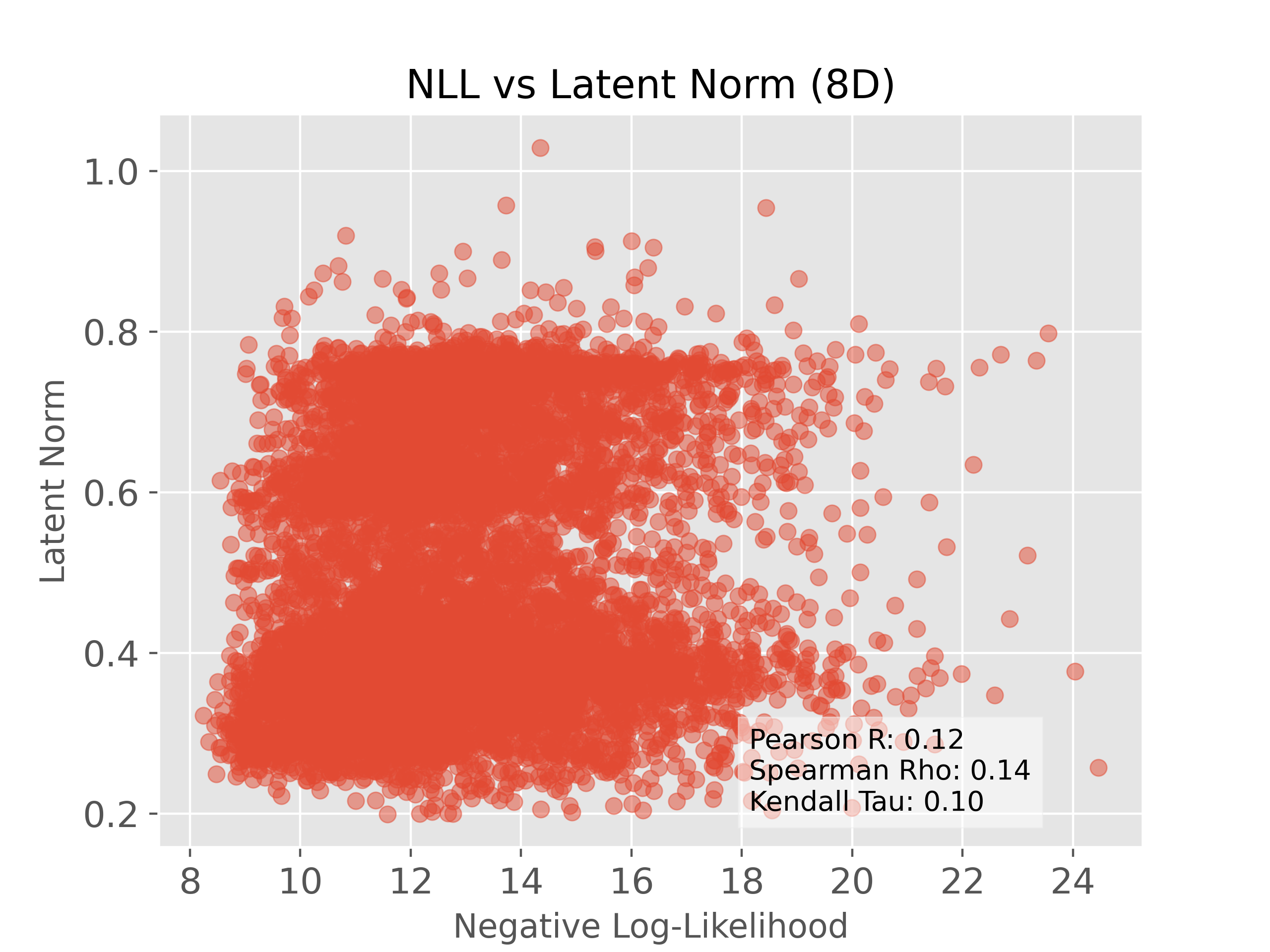}
    \caption{} \label{fig:gmm_nll_vs_norm_8d_deepsvdd}
  \end{subfigure}
  \begin{subfigure}[t]{0.24\textwidth}
    \includegraphics[width=\textwidth]{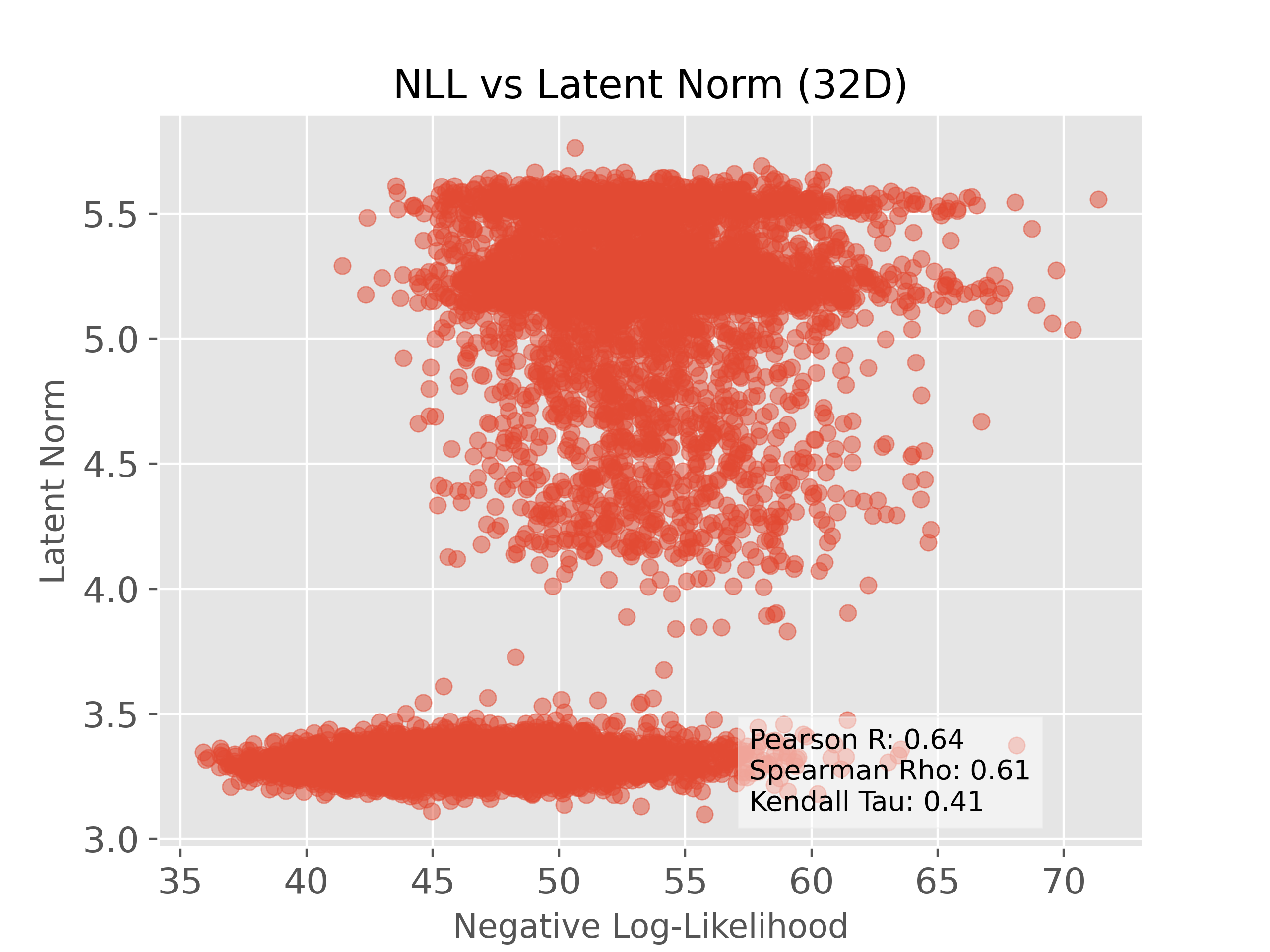}
    \caption{} \label{fig:gmm_nll_vs_norm_32d_deepsvdd}
  \end{subfigure}\hfill
  \begin{subfigure}[t]{0.24\textwidth}
    \includegraphics[width=\textwidth]{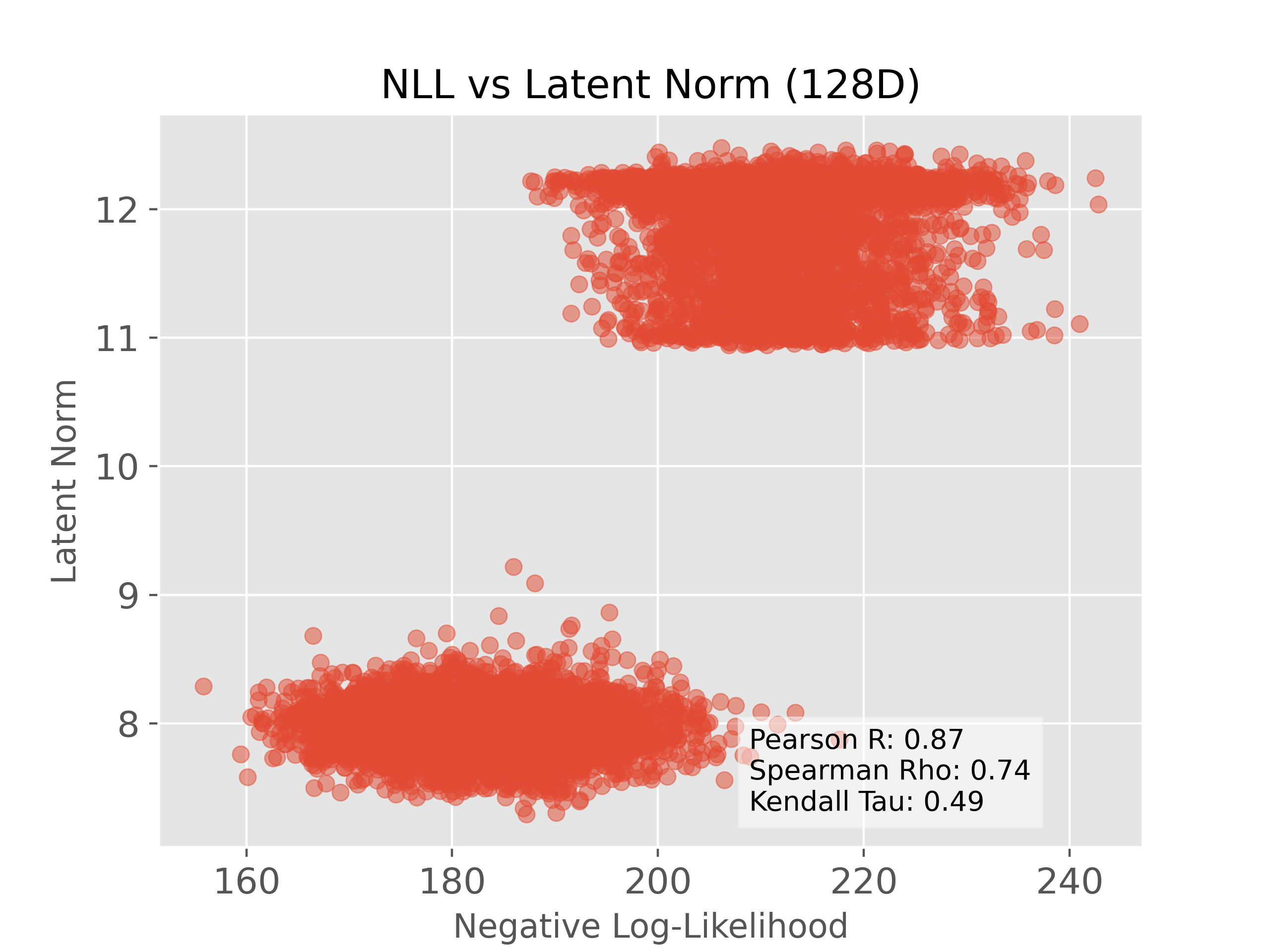}
    \caption{} \label{fig:gmm_nll_vs_norm_128d_deepsvdd}
  \end{subfigure}

    \begin{subfigure}[t]{0.24\textwidth}
    \includegraphics[width=\textwidth]{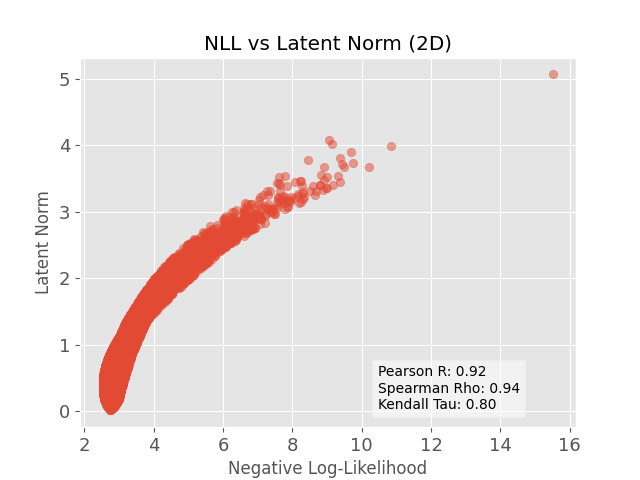}
    \caption{} \label{fig:gmm_nll_vs_norm_2d_nonus}
  \end{subfigure}\hfill
  \begin{subfigure}[t]{0.24\textwidth}
    \includegraphics[width=\textwidth]{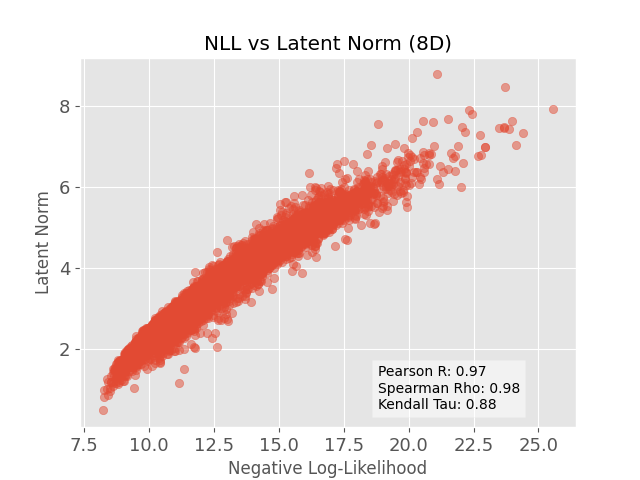}
    \caption{} \label{fig:gmm_nll_vs_norm_8d_nonus}
  \end{subfigure}
  \begin{subfigure}[t]{0.24\textwidth}
    \includegraphics[width=\textwidth]{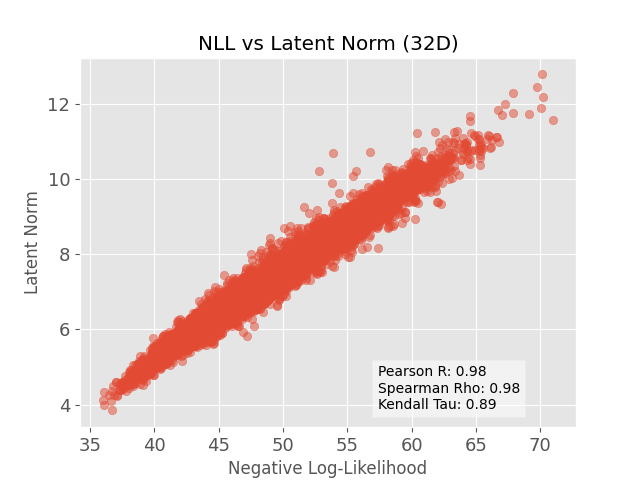}
    \caption{} \label{fig:gmm_nll_vs_norm_32d_nonus}
  \end{subfigure}\hfill
  \begin{subfigure}[t]{0.24\textwidth}
    \includegraphics[width=\textwidth]{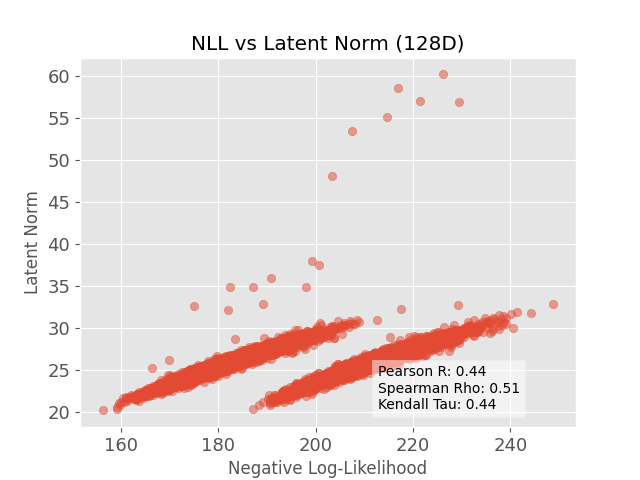}
    \caption{} \label{fig:gmm_nll_vs_norm_128d_nonus}
  \end{subfigure}
  
  \begin{subfigure}[t]{0.24\textwidth}
    \includegraphics[width=\textwidth]{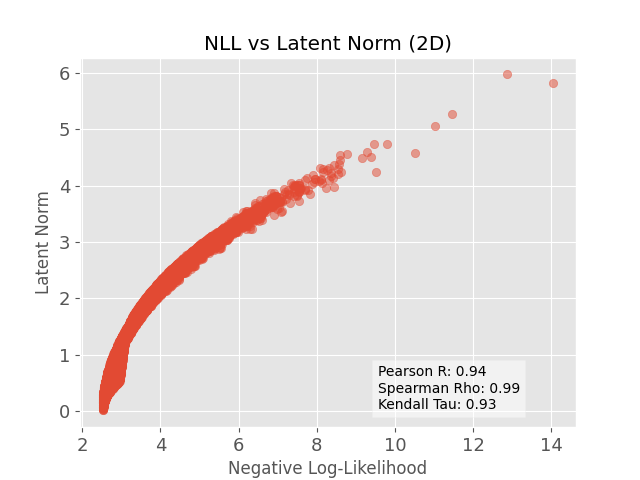}
    \caption{} \label{fig:gmm_nll_vs_norm_2d}
  \end{subfigure}\hfill
  \begin{subfigure}[t]{0.24\textwidth}
    \includegraphics[width=\textwidth]{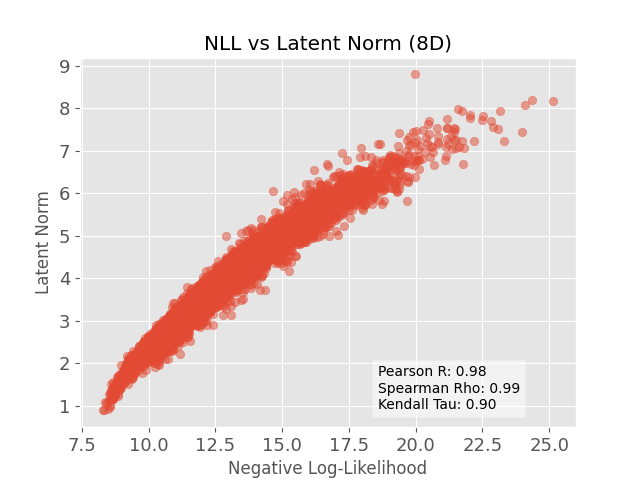}
    \caption{} \label{fig:gmm_nll_vs_norm_8d}
  \end{subfigure}
  \begin{subfigure}[t]{0.24\textwidth}
    \includegraphics[width=\textwidth]{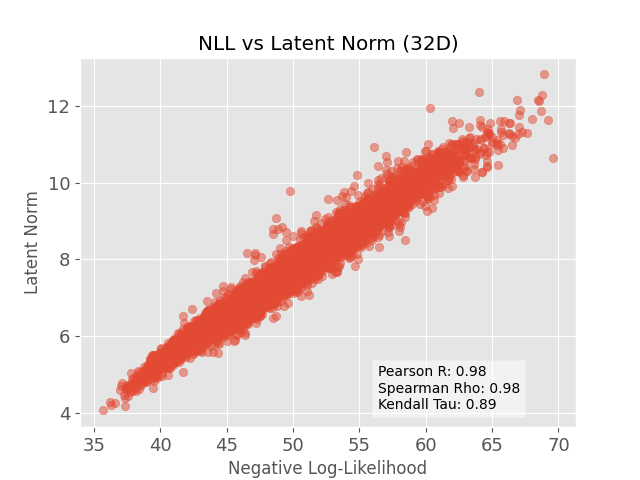}
    \caption{} \label{fig:gmm_nll_vs_norm_32d}
  \end{subfigure}\hfill
  \begin{subfigure}[t]{0.24\textwidth}
    \includegraphics[width=\textwidth]{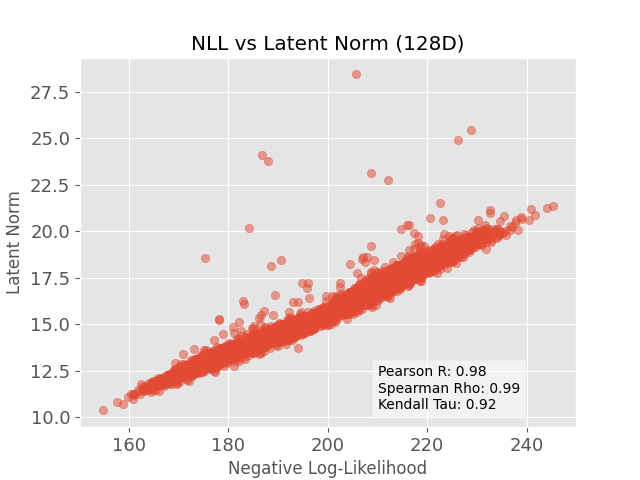}
    \caption{} \label{fig:gmm_nll_vs_norm_128d}
  \end{subfigure}

  \caption{%
    Scatter plots of the true log-likelihoods against the latent norms for the GM experiments (2, 8, 32, 128 dimensions).
    (a) -- (d): Deep SVDD
    (e) -- (h): Non-USF
    (i) -- (l) USF.
    USFs maintain a monotonic relationship between latent norm and data density. Non-USFs follow this behavior at first, but the relationship collapses at 128 dimensions.
    Deep SVDD, by contrast, transitions from a seemingly random latent structure to distinct homogeneous clusters as dimensionality increases.
  }
  \label{fig:gauss-mapping-all-norm}
\end{figure}

\begin{figure}[htbp!]
  \centering
  \captionsetup[subfigure]{labelformat=parens,justification=centering}
  % first row
  \begin{subfigure}[t]{0.24\textwidth}
    \includegraphics[width=\textwidth]{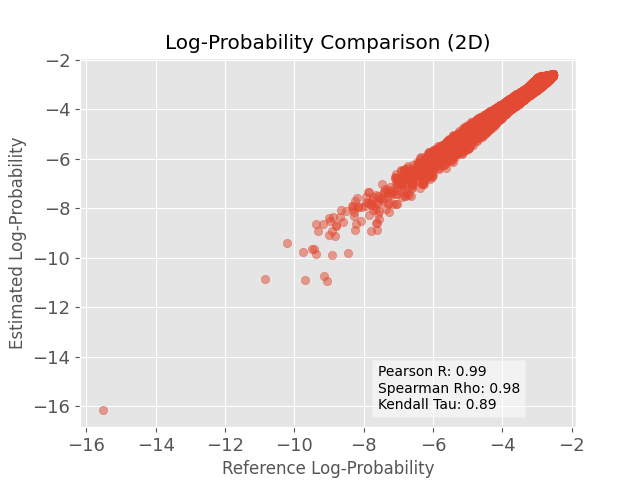}
    \caption{} \label{fig:gmm_nll_vs_norm_2d_nonus}
  \end{subfigure}\hfill
  \begin{subfigure}[t]{0.24\textwidth}
    \includegraphics[width=\textwidth]{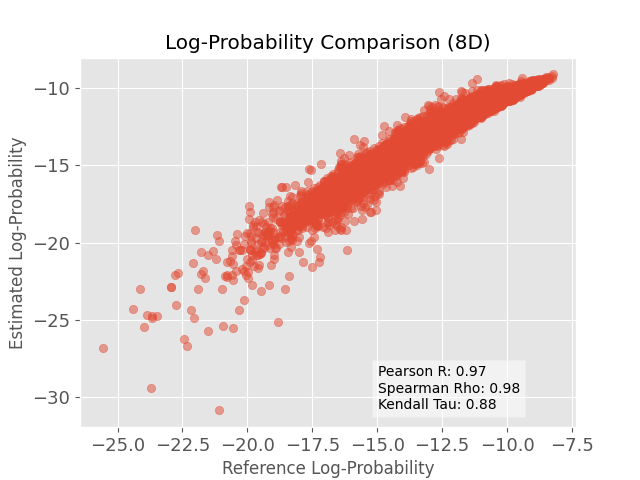}
    \caption{} \label{fig:gmm_nll_vs_norm_8d_nonus}
  \end{subfigure}
  \begin{subfigure}[t]{0.24\textwidth}
    \includegraphics[width=\textwidth]{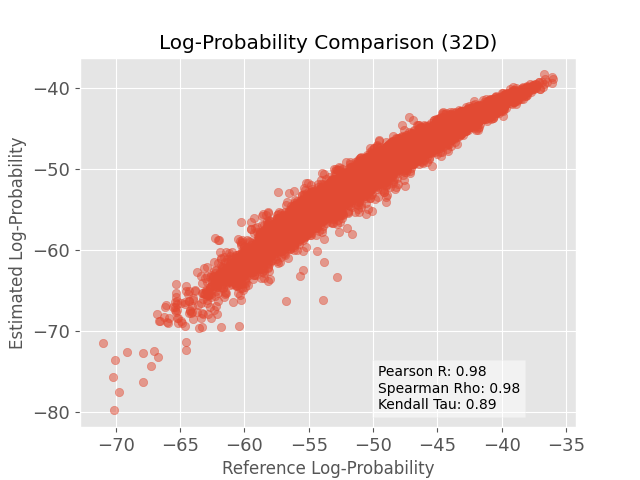}
    \caption{} \label{fig:gmm_nll_vs_norm_32d_nonus}
  \end{subfigure}\hfill
  \begin{subfigure}[t]{0.24\textwidth}
    \includegraphics[width=\textwidth]{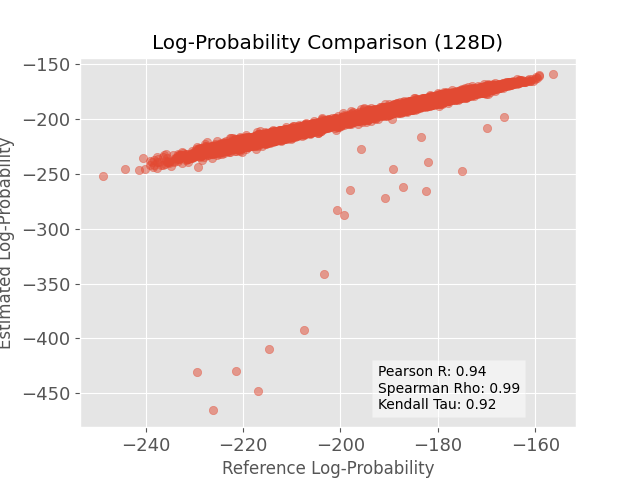}
    \caption{} \label{fig:gmm_nll_vs_norm_128d_nonus}
  \end{subfigure}

    \begin{subfigure}[t]{0.24\textwidth}
    \includegraphics[width=\textwidth]{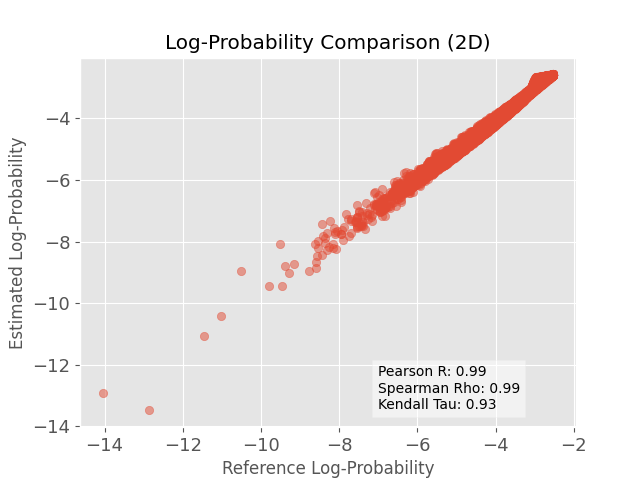}
    \caption{} \label{fig:gmm_nll_vs_norm_2d}
  \end{subfigure}\hfill
  \begin{subfigure}[t]{0.24\textwidth}
    \includegraphics[width=\textwidth]{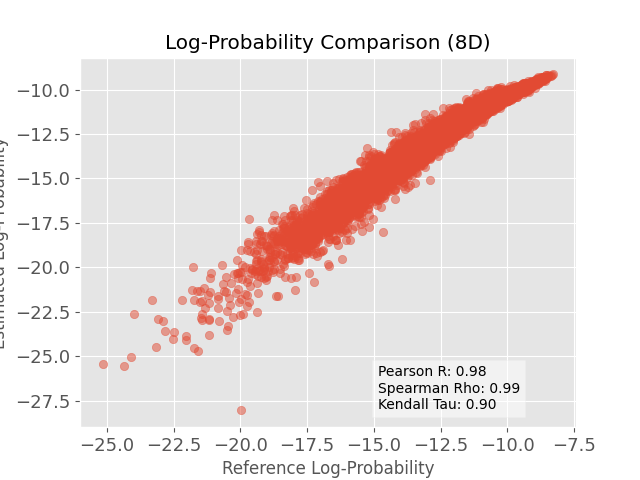}
    \caption{} \label{fig:gmm_nll_vs_norm_8d}
  \end{subfigure}
  \begin{subfigure}[t]{0.24\textwidth}
    \includegraphics[width=\textwidth]{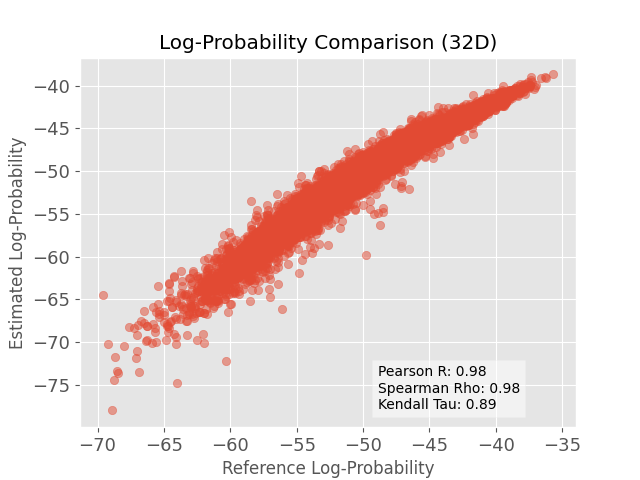}
    \caption{} \label{fig:gmm_nll_vs_norm_32d}
  \end{subfigure}\hfill
  \begin{subfigure}[t]{0.24\textwidth}
    \includegraphics[width=\textwidth]{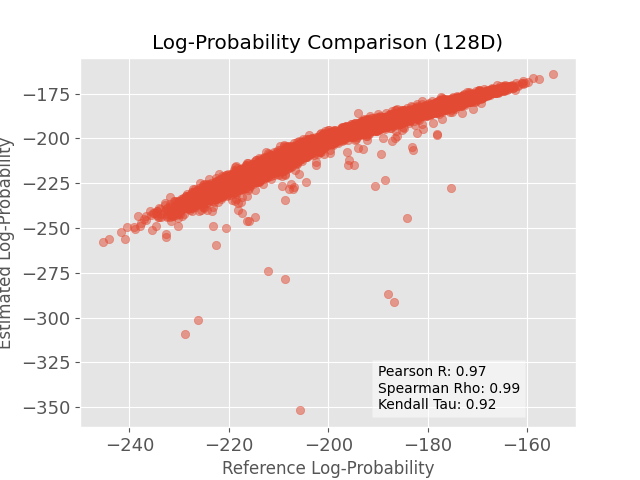}
    \caption{} \label{fig:gmm_nll_vs_norm_128d}
  \end{subfigure}

  \caption{%
    Scatter plots of the true log-likelihoods against the estimated log-likelihoods for the GM experiments (2, 8, 32, 128 dimensions).
    (a) -- (d): Non-USF
    (e) -- (h): USF.
    Across increasing dimensionality, the estimated likelihood quality remains largely comparable between both flow models.
  }
  \label{fig:gauss-mapping-all-logprob}
\end{figure}

\subsection{Dimensionality Reduction}\label{subsec:dim_reduction}

A key distinction between Deep SVDD and uniformly scaling flows (USFs) is that the diffeomorphism property of a USF requires its latent space to have the same dimensionality as the input data. However, a USF with a normal base distribution implicitly defines an effective reduction to a one-dimensional score: the latent norm $\phi'_w(x) := \|\phi_w(x) - c\|$ produces an anomaly ranking identical to that of the full likelihood, effectively acting as a Deep SVDD objective with a specialized regularizer (cf. Eq. \ref{eqn:usflow_deepsvdd_loss}).

To enable explicit dimensionality reduction, we propose a hybrid architecture, VAE-USFlow, which integrates a USF as a learnable prior within a variational autoencoder (VAE) framework. This combines the reconstruction-based objective of a VAE with the density-estimation power of a USF, drawing inspiration from joint training techniques in Deep SVDD~\cite{HojjatiA24} and prior learning in VAEs~\cite{Tomczak2018}. This approach encourages a meaningful compressed representation. We note that the beneficial properties of the USF apply only to this encoder-produced latent space. To mitigate the known issue of autoencoders potentially mapping out-of-distribution data into high-density regions of the prior~\cite{ramakrishna2022efficient}, we combine the latent likelihood with a reconstruction score. To the best of our knowledge, this constitutes a novel application of training a flow-based prior within the VAE paradigm for anomaly detection.

\paragraph{VAE-USFlow} Before we outline the architecture, we briefly review the original VAE formulation: Consider a generative model with latent variables $\hat{p}_X(x) = \int \hat{p}_{X|Z}(x\mid z)p_Z(z)  dz$ together with an amortized posterior approximation $\hat{p}_{Z|X}(z\mid x)$. The system is optimized by maximizing the expected evidence lower bound (ELBO) over the data distribution:
\begin{align}
    \textrm{ELBO}(x; \hat{p}_{Z|X}\| \hat{p}_{Z|X}) &= \ln \hat{p}_X(x) - \mathrm{KL}(\hat{p}_{Z|X}(z\mid x) \| \hat{p}_{Z|X}(z\mid x)) \nonumber \\
    &= \EE_{z\sim \hat{p}_{Z|X}(z\mid x)}\left[\ln \hat{p}_{X|Z}(x \mid z) - \ln \hat{p}_{Z|X}(z\mid x) + \ln p_Z(z)\right] \label{eqn:elbo} \\
    &= \EE_{z\sim \hat{p}_{Z|X}(z\mid x)}\left[\ln \hat{p}_{X|Z}(x \mid z)\right] - \mathrm{KL}(\hat{p}_{Z|X}(z\mid x) \| p_Z(z)). \label{eqn:vamprior}
\end{align}

Typically, the latent prior $p_Z(z)$ is fixed as a simple distribution, e.g. $\mathcal{N}(0, I)$. From Equation \ref{eqn:vamprior} we see that when $p_Z(z)$ is learnable, maximizing the expected ELBO adapts the prior to the expected variational posterior, which is the true prior under the joint $\hat{p}_{X,Z}(x,z) = p_X(x)\hat{p}_{Z\mid X}(z|x)$ \cite{Tomczak2018}. 

We are going to decompose our task of aligning the data in a compressed latent space into two tasks. The encoder (amortized variational posterior) encodes the instances in the compressed space, and a prior given by a trainable USF further aligns the densities with the norms in a secondary latent space. The ELBO objective avoids collapse in the primary space while the flow prior grants more flexibility to the feature encoder.

Our joint model $\hat{p}^{\zeta,\eta}_{X,Z}(x, z) = \hat{p}^\zeta_{X\mid Z}(x \mid z) p^\eta_Z(z)$ uses the following generative process:
\begin{align*}
    B_{\text{prior}} &\sim \mathcal{N}(0, I), \\ 
    Z &\sim \phi^{-1}_\eta(B_{\text{prior}}), \\
    X\mid Z &\sim \mathcal{N}(\mathrm{dec}_\zeta(z), \sigma_{\min}^2 I),
\end{align*}
with $\sigma_{\min} > 0$ sufficiently small. The variational posterior approximation $\hat{p}^\eta_{Z\mid X}(z\mid x)$ is defined as:
\begin{align*}
  Z \mid X &\sim \mathcal{N}(\mathrm{enc}^\mu_\theta(x), \mathrm{enc}^\Sigma_\theta(x)) 
\end{align*}

Using Equation~\ref{eqn:elbo}, we optimize the model end-to-end via:
\begin{align*}
\max_{\zeta, \eta, \theta} \bb{E}_{x\sim \mathcal{D}} \bb{E}_{z\sim q_\theta(z\mid x)} \Biggl[ 
  &\ln \mathcal{N}\left(x; \mathrm{dec}_\zeta(z), \sigma_{\min}^2 I\right) \\
  &- \ln \mathcal{N}\left(z; \mathrm{enc}^\mu_\theta(x), \mathrm{enc}^\Sigma_\theta(x)\right) \\
  &+ \ln \mathcal{N}\left(\phi_\eta(z); 0, I\right) + \ln\left| \det \tfrac{\partial \phi_\eta}{\partial z} (z) \right| \Biggr].
\end{align*}

\paragraph{Proof of Concept} We provide an implementation of the proposed architecture, which we call VAE-USFlow. The model is based on wide-ResNet50 encoder (initialized with a pretrained weights) with a final fully connected layer to project to the desired latent dimension and a simple decoder, which is based on transposed convolutions. The number of layers of the decoder is based on the compression ratio of latent space. As flow prior, we employ the USFlows architecture and use the NonUSFlows architecture for baseline comparison (VAE-NonUSFlow, with softplus as affine activation and clamping at 2.0, analogous to Anomalib defaults). We train both variants on the MVTec datasets. Figure \ref{fig:vae_results} shows the mean image AUC-ROC performance and standard deviation over three runs for the individual classes.

As a proof of concept, we evaluate VAE–USFlow alongside a version with a NonUSFlow prior on MVTec AD. 
An evaluation is given in Figure \ref{fig:vae_results}. While these hybrids do not yet match specialized flow-based anomaly detectors, the USF version achieves a substantially higher average performance and run-to-run performance stability relative to its affine counterpart. Building on these findings, our main empirical study focuses on drop-in USF substitutions within state-of-the-art anomaly detection architectures.

\begin{figure}[htp!]
  \centering
  \includegraphics[width=\textwidth]{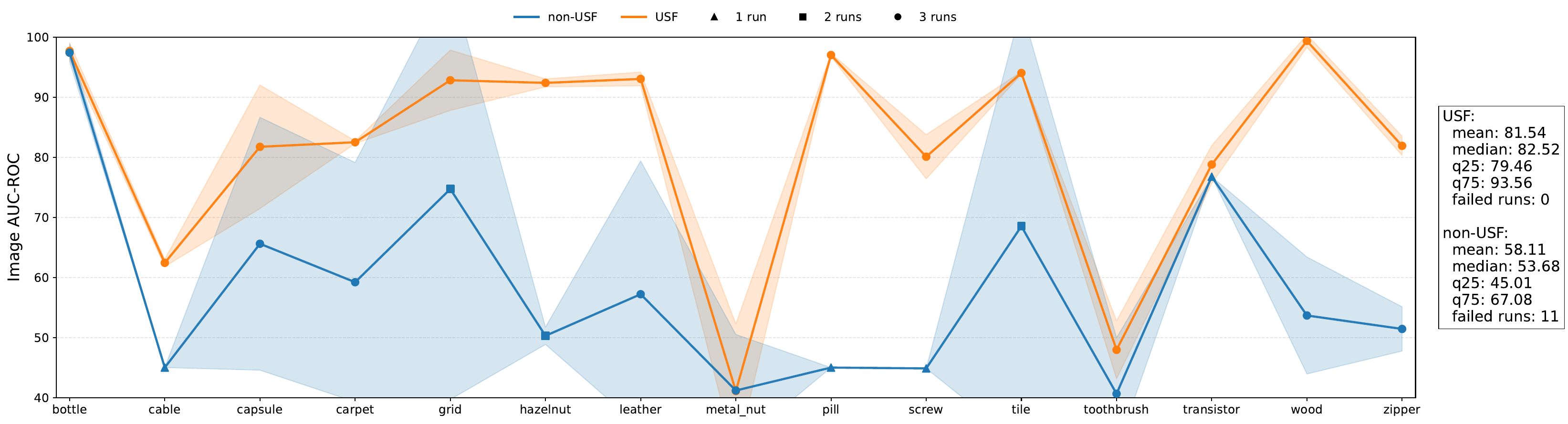}
  \caption{MVTec AD per-class mean image AUC-ROC (\%) $\pm$ stdev. for the VAE-Hybrid setting. USF consistently achieves higher average performance and markedly improved run-to-run performance stability. Only for non-USF we frequently encountered training failures due to numerical instabilities (see Section \ref{sup:nonusf_instabilities}). We indicate the number of valid runs (max 3) per class with distinct markers.}
  \label{fig:vae_results}
\end{figure}

\section{Experiments}\label{sec:experiments}

The objective-level link to Deep SVDD, most notably the spherical alignment of density level sets in the latent space, paired with the observations from our VAE experiments suggest, that USFs may serve as beneficial drop-in replacements within state-of-the art flow-based anomaly detection. 
Recent architectures such as FastFlow~\cite{YuZWLWTW21}, CFlow~\cite{GudovskiyIK22}, and U-Flow~\cite{TailanianPM24} achieve strong performance and are available via the unified Anomalib library~\cite{AkcayAVLAG22}, yet they rely on affine coupling layers that are not uniformly scaling.

To empirically evaluate our hypothesis, we instantiate a USF layer by (1) replacing affine coupling with \emph{additive} coupling (rendering the flow uniformly scaling) and (2) inserting LU-parameterized bijective affine transforms with fixed diagonal magnitudes to recover per-block expressivity.
This US variant is evaluated against its affine (non-US) counterpart across the MVTec AD~\cite{BergmannFSS19} and VisA~\cite{ZouJPZD22} benchmark datasets based on image‐level and pixel-level AUC‐ROC performance, using the three aforementioned model architectures.

\subsection{Experimental Settings}

% \paragraph{Tuning policy and comparability}
To isolate the practical impact of substituting affine (non-US) flows with USFs in unsupervised, out-of-the-box deployment, we adopt the architecture and hyperparameter choices from the original papers and refrain from model- or class-specific tuning. As a consequence, absolute performance values are lower than leaderboard reports. However, the within-architecture comparisons are strictly controlled and therefore reflect the causal effect of the flow substitution under various conditions, rather than tuning artifacts. The parameterization of the specific architectures is provided in the following.

\paragraph{CFlow}
Feature maps are drawn from a Wide ResNet-50-2 backbones layers 2–4. The conditional normalizing flow consists of eight coupling blocks, each conditioned on a learned 128-dimensional positional embedding. To prevent numerical instability, the coupling network’s log-scale outputs are clipped to $\pm1.9$ before exponentiation. The learning rate is $1\times 10^{-4}$ without weight decay.

\paragraph{FastFlow}
A ResNet-18 backbone feeds into eight alternating $1\times1$ and $3\times3$ convolutional coupling layers. The internal scale-and-shift subnetworks maintain the same channel width as their inputs, balancing expressivity with efficiency.  Training uses a learning rate of $1\times 10^{-3}$ and weight decay $1\times 10^{-5}$.

\paragraph{U-Flow}
A MS-CaiT backbone produces a U-shaped, multi-scale feature pyramid. Four 2D affine coupling steps at each scale use full-width subnet channels to preserve detailed spatial information. Log-scale predictions are clamped to $\pm2.0$ for stability. The initial learning rate is $1\times 10^{-3}$ with weight decay $1\times 10^{-5}$ and a linear decay to $40\%$ over $25{,}000$ iterations.

All experiments were conducted based on Anomalib v2.0.0, using its PyTorch Lightning implementations. Each model is trained for up to 100 epochs with batch size 32 using Adam, with early stopping after three epochs of no improvement. The feature extractor backbones are initialized and frozen with ImageNet-pretrained weights. Every configuration is run three times and we report mean~$\pm$~stdev. 

\subsection{Experimental Results}\label{sec:results}

\begin{table}[ht]
    \centering
    \begin{adjustbox}{width=\textwidth}
    \begin{tabular}{lccc ccc}
        \toprule
        \multicolumn{1}{c}{} 
        & \multicolumn{3}{c}{\textsc{Base Flow}} 
        & \multicolumn{3}{c}{\textsc{USF}} \\
        \cmidrule(lr){2-4}\cmidrule(lr){5-7}
        Dataset (Metric) 
        & FastFlow & U\,-Flow & CFlow
        & FastFlow & U\,-Flow & CFlow \\
        \midrule
        MVTec AD (Image) 
            & 91.3$\pm$1.5 &\underline{92.7$\pm$4.6} & 72.5$\pm$12.8 
            & 91.0$\pm$0.9 &\B 94.5$\pm$0.6 & 90.8$\pm$1.1 \\
        MVTec AD (Pixel) 
            & 95.8$\pm$0.4 & 95.9$\pm$2.8 & 95.4$\pm$1.0 
            & \underline{96.8$\pm$0.2} &\B 97.0$\pm$0.4 & 96.6$\pm$0.0 \\
        VisA (Image) 
            & \underline{88.0$\pm$1.7} & 84.6$\pm$6.7 & 82.5$\pm$8.3 
            & 87.7$\pm$0.6 & 85.6$\pm$0.6 &\B 88.4$\pm$0.9 \\
        VisA (Pixel) 
            & 96.3$\pm$0.7 & 95.8$\pm$3.8 & 97.8$\pm$0.5 
            & \B 98.8$\pm$0.0 & 95.7$\pm$1.8 & \underline{98.2$\pm$0.0} \\
        \bottomrule
    \end{tabular}
    \end{adjustbox}
    \caption{Mean AUC-ROC and stdev. (\%) over three seeds; dataset values are class-averaged. Highest mean per row in bold, second-highest underlined. Class-specific results are provided in Tables~\ref{tab:visa_allclasses_roc_auc}-~\ref{tab:mvtec_allclasses_pixel_roc_auc}.}
    \label{tab:condensed_main_avgs}
\end{table}

\paragraph{Overview}
Across datasets and granularities, replacing the affine base with uniformly scaling flows (USFs) either improves mean performance or leaves it effectively unchanged, while consistently and substantially reducing run-to-run variance. The effect is stable at both image and pixel level.

At a high level, USFs make the strongest impact where base flows are less stable: they markedly tighten variability (often approaching negligible standard deviations) and lift means most for CFlow and U-Flow, while FastFlow typically maintains its mean and still benefits from improved stability. On MVTec AD, USFs yield the most pronounced consistency gains and deliver the largest mean increase for CFlow; on VisA, the best row-wise means shift from U-Flow toward CFlow (image) and FastFlow (pixel) under USFs, with U-Flow remaining competitive and considerably more stable. Overall, the pattern is uniform: mean performances do not regress, and run-to-run variance contracts substantially.

\begin{table}[htbp!]
    \centering
    \begin{adjustbox}{width=\textwidth}
    \begin{tabular}{lccc ccc}
        \toprule
        \multicolumn{1}{c}{} 
        & \multicolumn{3}{c}{\textsc{Base Flow}} 
        & \multicolumn{3}{c}{\textsc{USF}}  \\
        \cmidrule(lr){2-4}\cmidrule(lr){5-7}
        Class 
        & FastFlow         & U-Flow            & CFlow
        & FastFlow         & U-Flow            & CFlow          \\
        \midrule
        candle  & \B 96.5$\pm$0.3   & 92.7$\pm$1.5   & 91.2$\pm$2.8   & 84.9$\pm$1.1   & 84.1$\pm$2.0   & \underline{94.6$\pm$0.4}   \\
        capsules  & 73.9$\pm$1.9   & 74.2$\pm$6.5   & \underline{77.6$\pm$14.8}   & 75.5$\pm$0.4   & \B 79.7$\pm$0.8   & 70.8$\pm$5.4   \\
        cashew  & 88.4$\pm$2.2   & 93.7$\pm$1.5   & \B 96.7$\pm$1.2   & 90.6$\pm$0.7   & 94.7$\pm$0.5   & \underline{95.9$\pm$0.6}   \\
        chewinggum  & 98.2$\pm$0.3   & 98.5$\pm$0.5   & \B 99.6$\pm$0.1   & \underline{99.4$\pm$0.1}   & 98.6$\pm$0.3   & \B 99.6$\pm$0.1   \\
        fryum  & \underline{89.5$\pm$5.0}   & 88.5$\pm$1.5   & 52.2$\pm$18.6   & \B 92.8$\pm$0.5   & 86.0$\pm$0.2   & 84.9$\pm$0.5   \\
        macaroni1  & 87.8$\pm$2.9   & 79.8$\pm$22.2   & 71.4$\pm$10.7   & \underline{88.8$\pm$0.5}   & 77.6$\pm$0.3   & \B 89.3$\pm$0.2   \\
        macaroni2  & 67.8$\pm$2.5   & 63.5$\pm$13.6   & \underline{78.6$\pm$18.2}   & \B 80.5$\pm$0.5   & 74.8$\pm$1.2   & 71.6$\pm$1.4   \\
        pcb1  & \underline{87.9$\pm$1.4}   & 76.8$\pm$15.2   & 87.8$\pm$5.4   & 83.8$\pm$1.4   & 83.1$\pm$0.2   & \B 92.7$\pm$0.4   \\
        pcb2  & \B 87.7$\pm$0.8   & 80.8$\pm$7.1   & 82.9$\pm$4.7   & 85.7$\pm$0.9   & 76.6$\pm$0.4   & \underline{86.3$\pm$0.8}   \\
        pcb3  & \B 86.9$\pm$1.3   & 80.5$\pm$8.8   & 71.6$\pm$11.5   & 81.4$\pm$0.3   & \underline{85.5$\pm$1.2}   & 79.5$\pm$0.6   \\
        pcb4  & \underline{96.2$\pm$0.2}   & 90.4$\pm$1.8   & 96.1$\pm$1.2   & 92.5$\pm$0.6   & 88.7$\pm$0.5   & \B 97.3$\pm$0.4   \\
        pipe\_fryum  & 94.9$\pm$1.7   & 95.3$\pm$0.6   & 84.8$\pm$9.9   & 96.7$\pm$0.1   & \underline{97.9$\pm$0.1}   & \B 98.6$\pm$0.1   \\
        \midrule
        Average  
            &  88.0$\pm$1.7  & 84.6$\pm$6.7   & 82.5$\pm$8.3   
            & 87.7$\pm$0.6   & 85.6$\pm$0.6   & 88.4$\pm$0.9   \\
        \midrule
        \#Best  
            & 3   & 0   & 2   & 2   & 1   & 5 \\
        \bottomrule
    \end{tabular}
    \end{adjustbox}
    \caption[VisA  ROC-AUC by class]{Mean and stdev. for image  AUC-ROC (\%) on VisA over 3 runs. Highest value per row in bold, second-highest underlined. “\#Best” counts how often each column had the highest value.}
    \label{tab:visa_allclasses_roc_auc}
\end{table}

\begin{table}[htbp!]
    \centering
    \begin{adjustbox}{width=\textwidth}
    \begin{tabular}{lccc ccc}
        \toprule
        \multicolumn{1}{c}{} 
        & \multicolumn{3}{c}{\textsc{Base Flow}} 
        & \multicolumn{3}{c}{\textsc{USF}}  \\
        \cmidrule(lr){2-4}\cmidrule(lr){5-7}
        Class 
        & FastFlow         & U-Flow            & CFlow
        & FastFlow         & U-Flow            & CFlow          \\
        \midrule
        candle  & 98.0$\pm$0.4   & \underline{99.0$\pm$0.1}   & \underline{99.0$\pm$0.0}   & \B 99.1$\pm$0.0   & 94.8$\pm$3.5   & \underline{99.0$\pm$0.0}   \\
        capsules  & 98.0$\pm$0.1   & 95.7$\pm$3.5   & 96.5$\pm$0.7   & \B 99.2$\pm$0.0   & \underline{98.2$\pm$0.1}   & 97.2$\pm$0.0   \\
        cashew  & 98.2$\pm$0.3   & \B 99.4$\pm$0.1   & 98.3$\pm$0.3   & 98.8$\pm$0.0   & \underline{99.3$\pm$0.1}   & 98.4$\pm$0.0   \\
        chewinggum  & 98.9$\pm$0.1   & 99.2$\pm$0.1   & 99.0$\pm$0.2   & \B 99.4$\pm$0.0   & \underline{99.3$\pm$0.2}   & \B 99.4$\pm$0.0   \\
        fryum  & 84.6$\pm$4.6   & 95.5$\pm$0.5   & \B 97.1$\pm$0.8   & \B 97.1$\pm$0.0   & 96.3$\pm$0.1   & \underline{97.0$\pm$0.1}   \\
        macaroni1  & 97.8$\pm$0.7   & 94.0$\pm$9.1   & 97.3$\pm$0.6   & \B 99.7$\pm$0.0   & 88.3$\pm$8.4   & \underline{99.1$\pm$0.0}   \\
        macaroni2  & 94.5$\pm$0.3   & 81.9$\pm$20.2   & 97.1$\pm$1.3   & \B 99.1$\pm$0.0   & 82.6$\pm$8.0   & \underline{98.3$\pm$0.1}   \\
        pcb1  & 99.1$\pm$0.1   & 95.1$\pm$7.1   & \underline{99.3$\pm$0.1}   & 99.2$\pm$0.0   & \B 99.4$\pm$0.0   & \underline{99.3$\pm$0.0}   \\
        pcb2  & 96.7$\pm$0.4   & 97.2$\pm$0.7   & \underline{97.4$\pm$0.4}   & \B 98.3$\pm$0.1   & 96.3$\pm$0.3   & 96.7$\pm$0.0   \\
        pcb3  & 96.6$\pm$0.2   & 97.2$\pm$0.9   & 97.0$\pm$0.4   & \B 98.7$\pm$0.0   & 96.3$\pm$1.4   & \underline{97.3$\pm$0.0}   \\
        pcb4  & 97.5$\pm$0.3   & 96.2$\pm$2.5   & 97.2$\pm$0.6   & 97.2$\pm$0.1   & \B 98.5$\pm$0.0   & \underline{97.7$\pm$0.1}   \\
        pipe\_fryum  & 96.0$\pm$1.0   & 98.9$\pm$0.6   & 98.8$\pm$0.1   & \B 99.2$\pm$0.0   & \B 99.2$\pm$0.0   & \underline{99.0$\pm$0.0}   \\
        \midrule
        Average  
            & 96.3$\pm$0.7  & 95.8$\pm$3.8   & 97.8$\pm$0.5   
            & 98.8$\pm$0.0  & 95.7$\pm$1.8   & 98.2$\pm$0.0   \\
        \midrule
        \#Best  
            & 0   & 1   & 1   & 9   & 3   & 1   \\
        \bottomrule
    \end{tabular}
    \end{adjustbox}
    \caption[VisA  Pixel ROC-AUC by class]{Mean and stdev. for pixel  AUC-ROC (\%) on VisA over 3 runs. Highest value per row in bold, second-highest underlined. “\#Best” counts how often each column had the highest value.}
    \label{tab:visa_allclasses_pixel_roc_auc}
\end{table}

\paragraph{VisA — image-level}
The aggregate trends for VisA carry over to the class-specific performances (Table~\ref{tab:visa_allclasses_roc_auc}).
CFlow sees large gains with a collapse in variance (\emph{fryum} $52.2\pm18.6\!\to\!84.9\pm0.5$, \emph{pcb1} $87.8\pm5.4\!\to\!92.7\pm0.4$), lifting the dataset average from $82.5$ to $88.4$.
U-Flow improves modestly ($84.6\!\to\!85.6$) with a strong stability gain ($6.7\!\to\!0.6$).
FastFlow remains essentially unchanged in mean ($88.0\!\to\!87.7$) but becomes more stable ($1.7\!\to\!0.6$).
USF variants secure more per-class wins (8 vs.\ 5 overall).

\paragraph{VisA — pixel-level}
USF enhances or preserves mean performance and reduces variance (Table~\ref{tab:visa_allclasses_pixel_roc_auc}).
FastFlow gains strongly ($96.3\!\to\!98.8$; stdev $\to 0$), CFlow improves ($97.8\!\to\!98.2$; stdev $\to 0$), and U-Flow is stable in mean with reduced variance ($95.8\!\to\!95.7$, $3.8\!\to\!1.8$).
Illustrative per-class lifts include \emph{macaroni2} for FastFlow ($94.5\!\to\!99.1$) and \emph{pcb3} ($96.6\!\to\!98.7$).
Across classes, USF achieves 13 best scores vs.\ only 2 for the affine counterparts.

\paragraph{MVTec AD — image-level}
Class-wise results (Table~\ref{tab:mvtec_allclasses_roc_auc}) show that USF removes extreme volatility and lifts accuracy in most difficult cases.
For CFlow, significant volatilities virtually disappear while accuracy jumps dramatically (e.g., \emph{metal\_nut} $32.6\pm42.6\!\to\!99.3\pm0.2$, \emph{toothbrush} $35.3\pm26.1\!\to\!96.9\pm1.5$).
U-Flow benefits in both mean and stability (dataset average $92.7\!\to\!94.5$, stdev $4.6\!\to\!0.6$) and achieves the most per-class wins (9/15 vs.\ 5/15 for the affine variant).
FastFlow maintains its mean (91.3 vs.\ 91.0) but still gains per-class robustness and occasionally large improvements (e.g., \emph{screw} $+16.2$ in mean performance).

\begin{table}[htbp!]
    \centering
    \begin{adjustbox}{width=\textwidth}
    \begin{tabular}{lccc ccc}
        \toprule
        \multicolumn{1}{c}{} 
        & \multicolumn{3}{c}{\textsc{Base Flow}} 
        & \multicolumn{3}{c}{\textsc{USF}}  \\
        \cmidrule(lr){2-4}\cmidrule(lr){5-7}
        Class 
        & FastFlow         & U-Flow            & CFlow
        & FastFlow         & U-Flow            & CFlow          \\
        \midrule
        bottle  
            & \B 100.0$\pm$0.0   & \underline{99.8$\pm$0.3}       & 90.7$\pm$8.9   
            & 98.9$\pm$0.2       & \B 100.0$\pm$0.0 & \B 100.0$\pm$0.0       \\
        cable  
            & \underline{91.8$\pm$0.8}  & 85.8$\pm$7.1       & 85.7$\pm$6.8   
            & 71.0$\pm$2.3       & \B 94.0$\pm$0.3         & 86.2$\pm$1.1       \\
        capsule  
            & \underline{88.9$\pm$1.7}  & \B 94.1$\pm$2.1     & 83.5$\pm$7.9   
            & 84.7$\pm$1.2       & 80.3$\pm$0.9           & 87.2$\pm$0.6       \\
        carpet  
            & 98.0$\pm$0.6        & \B 100.0$\pm$0.0    & 78.4$\pm$8.4   
            & 96.8$\pm$0.2       & \underline{99.9$\pm$0.0} & 95.9$\pm$1.3       \\
        grid  
            & 97.7$\pm$0.8        & \underline{98.8$\pm$0.5} & 97.7$\pm$1.5   
            & 95.0$\pm$0.8       & \B 99.3$\pm$0.9         & 72.1$\pm$1.1       \\
        hazelnut  
            & 78.0$\pm$5.7        & 97.7$\pm$1.2       & 65.4$\pm$26.4  
            & \underline{99.5$\pm$0.1} & \B 100.0$\pm$0.0      & \B 100.0$\pm$0.0      \\
        leather  
            & \B 100.0$\pm$0.0     & \B 100.0$\pm$0.0    & 96.4$\pm$5.1   
            & \underline{99.7$\pm$0.1} & \B 100.0$\pm$0.0      & \B 100.0$\pm$0.0      \\
        metal\_nut  
            & 97.1$\pm$1.4        & 82.7$\pm$29.9      & 32.6$\pm$42.6  
            & 95.5$\pm$0.2       & \B 100.0$\pm$0.0      & \underline{99.3$\pm$0.2}       \\
        pill  
            & 92.5$\pm$0.5        & \underline{94.1$\pm$2.9} & 84.5$\pm$4.0   
            & 88.2$\pm$1.0       & \B 96.3$\pm$0.4        & 84.3$\pm$3.7       \\
        screw  
            & 67.7$\pm$2.4        & 71.0$\pm$4.1       & \underline{83.8$\pm$10.7}   
            & \B 83.9$\pm$1.2     & 71.5$\pm$0.8           & 61.9$\pm$4.7       \\
        tile  
            & 97.8$\pm$0.1        & 99.9$\pm$0.1       & 40.3$\pm$8.5   
            & 94.1$\pm$0.9       & \B 100.0$\pm$0.0        & \underline{100.0$\pm$0.0}       \\
        toothbrush  
            & 72.3$\pm$5.5        & 86.5$\pm$9.6       & 35.3$\pm$26.1  
            & \underline{92.7$\pm$0.6} & 91.8$\pm$0.6         & \B 96.9$\pm$1.5       \\
        transistor  
            & \underline{94.3$\pm$1.9} & 84.5$\pm$9.1       & 72.7$\pm$20.1   
            & 84.9$\pm$0.6       & \B 95.7$\pm$1.1        & 86.1$\pm$2.4       \\
        wood  
            & 98.1$\pm$0.3        & \B 99.4$\pm$0.1     & 53.6$\pm$7.4   
            & 98.2$\pm$0.2       & 96.9$\pm$3.2           & \underline{99.1$\pm$0.1}       \\
        zipper  
            & \underline{95.9$\pm$0.1} & \B 95.9$\pm$1.7     & 86.2$\pm$7.3   
            & 81.9$\pm$3.2       & 92.0$\pm$1.1           & 92.8    $\pm$0.2       \\
        \midrule
        Average  
            &  91.3$\pm$1.5  & 92.7$\pm$4.6   & 72.5$\pm$12.8   
            & 91.0$\pm$0.9   & 94.5$\pm$0.6   & 90.8$\pm$1.1   \\
        \midrule
        \#Best  
            & 2   & 5   & 0   
            & 1   & 9   & 4   \\
        \bottomrule
    \end{tabular}
    \end{adjustbox}
    \caption[Unsupervised MVTec AD  ROC-AUC by class]{Mean and stdev. for image AUC-ROC (\%) on MVTec AD over 3 runs. Highest value per row in bold, second-highest underlined. “\#Best” counts how often each column had the highest value.}
    \label{tab:mvtec_allclasses_roc_auc}
\end{table}

\paragraph{MVTec AD — pixel-level}
USF yields consistent improvements for all three models (Table~\ref{tab:mvtec_allclasses_pixel_roc_auc}): U-Flow $95.9\!\to\!97.0$, CFlow $95.4\!\to\!96.6$, FastFlow $95.8\!\to\!96.8$, with strong variance reductions (e.g., U-Flow $2.8\!\to\!0.4$).
Per-class patterns mirror the image-level story: unstable affine cases stabilize and typically improve under USF—
e.g., \emph{metal\_nut} for U-Flow $80.2\pm28.8\!\to\!97.8\pm0.2$ and \emph{screw} for FastFlow $84.5\pm1.9\!\to\!98.3\pm0.1$.
Overall, USF columns produce twice as many best-per-class scores (12 vs.\ 6).

\begin{table}[htbp!]
    \centering
    \begin{adjustbox}{width=\textwidth}
    \begin{tabular}{lccc ccc}
        \toprule
        \multicolumn{1}{c}{} 
        & \multicolumn{3}{c}{\textsc{Base Flow}} 
        & \multicolumn{3}{c}{\textsc{USF}}  \\
        \cmidrule(lr){2-4}\cmidrule(lr){5-7}
        Class 
        & FastFlow         & U-Flow            & CFlow
        & FastFlow         & U-Flow            & CFlow          \\
        \midrule
        bottle  & 97.4$\pm$0.2   & 96.9$\pm$0.9   & 97.7$\pm$0.2   & 97.8$\pm$0.1   & \B 98.7$\pm$0.1   & \underline{98.3$\pm$0.0}   \\
        cable  & 94.9$\pm$0.2   & \underline{95.6$\pm$2.5}   & 94.9$\pm$1.4   & 93.8$\pm$1.1   & \B 97.2$\pm$0.1   & 95.2$\pm$0.1   \\
        capsule  & 97.7$\pm$0.6   & \B 98.8$\pm$0.0   & \underline{98.6$\pm$0.2}   & \B 98.8$\pm$0.0   & 98.1$\pm$0.0   & \B 98.8$\pm$0.0   \\
        carpet  & 98.6$\pm$0.1   & \underline{99.3$\pm$0.1}   & 98.8$\pm$0.1   & 98.7$\pm$0.0   & \B 99.5$\pm$0.0   & 98.8$\pm$0.0   \\
        grid  & \underline{98.0$\pm$0.4}   & 97.8$\pm$0.1   & 97.4$\pm$0.1   & \B 98.3$\pm$0.1   & 97.4$\pm$0.1   & 96.7$\pm$0.1   \\
        hazelnut  & 93.3$\pm$1.0   & \underline{98.9$\pm$0.2}   & 97.0$\pm$0.7   & 98.3$\pm$0.0   & \B 99.4$\pm$0.0   & 98.5$\pm$0.0   \\
        leather  & \B 99.6$\pm$0.0   & 99.3$\pm$0.1   & 99.4$\pm$0.1   & 98.9$\pm$0.2   & \underline{99.5$\pm$0.1}   & 99.1$\pm$0.0   \\
        metal\_nut  & 96.0$\pm$0.3   & 80.2$\pm$28.8   & 92.3$\pm$5.1   & 97.3$\pm$0.1   & \B 97.8$\pm$0.2   & \underline{97.5$\pm$0.1}   \\
        pill  & 96.9$\pm$0.4   & \B 98.2$\pm$0.9   & 97.2$\pm$0.4   & 97.7$\pm$0.0   & 97.8$\pm$0.1   & \underline{98.0$\pm$0.0}   \\
        screw  & 84.5$\pm$1.9   & 96.5$\pm$1.1   & \underline{97.5$\pm$0.4}   & \B 98.3$\pm$0.1   & 93.9$\pm$2.9   & 96.8$\pm$0.1   \\
        tile  & 94.4$\pm$0.3   & \underline{96.8$\pm$0.4}   & 94.2$\pm$0.7   & 92.1$\pm$0.2   & \B 97.0$\pm$0.3   & 95.8$\pm$0.0   \\
        toothbrush  & 96.3$\pm$0.3   & 97.8$\pm$1.6   & 97.6$\pm$0.3   & \underline{98.9$\pm$0.0}   & \B 99.0$\pm$0.1   & 98.4$\pm$0.0   \\
        transistor  & \B 96.4$\pm$0.4   & 87.5$\pm$5.2   & 84.6$\pm$2.2   & \underline{90.8$\pm$0.1}   & 87.3$\pm$0.9   & 85.9$\pm$0.2   \\
        wood  & 95.3$\pm$0.3   & \B 97.1$\pm$0.1   & 87.4$\pm$2.3   & 94.2$\pm$0.2   & \underline{96.7$\pm$0.4}   & 94.3$\pm$0.0   \\
        zipper  & \B 97.7$\pm$0.2   & \underline{97.5$\pm$0.2}   & 96.7$\pm$1.4   & \B 97.7$\pm$0.2   & 95.4$\pm$0.4   & 97.0$\pm$0.0   \\
        \midrule
        Average  
            &  95.8$\pm$0.4  & 95.9$\pm$2.8   & 95.4$\pm$1.0   
            & 96.8$\pm$0.2   & 97.0$\pm$0.4   & 96.6$\pm$0.0   \\
        \midrule
        \#Best  
            & 3   & 3   & 0   & 4   & 7   & 1   \\
        \bottomrule
    \end{tabular}
    \end{adjustbox}
    \caption[MVTec AD  Pixel ROC-AUC by class]{Mean and stdev. for pixel AUC-ROC (\%) on MVTec AD over 3 runs. Highest value per row in bold, second-highest underlined. “\#Best” counts how often each column had the highest value.}
    \label{tab:mvtec_allclasses_pixel_roc_auc}
\end{table}

\paragraph{Discussion}
The observation that USF variants significantly improve or maintain performance and systematically lower run-to-run variance across architectures and datasets aligns with our hypothesis: the input-dependent log-det degree of freedom in affine flows is unnecessary --- and often harmful --- for unsupervised anomaly detection in these backbones.
Removing it re-focuses learning on aligning latent norms with the (isotropic) base and improves numerical conditioning.
Notably, this effect arises even though our replacement relies on additive coupling, which is strictly less expressive per block than the affine alternative~\cite{Draxler_universality_2024}; empirically, the improved inductive bias and stability outweigh the local expressivity gap.
While we observe occasional per-class regressions, the aggregate picture is clear: USF is a drop-in change that yields higher accuracy (especially for CFlow and pixel-level FastFlow) and markedly better reliability with far fewer unstable runs.

\subsection{Ablation Study}\label{sec:ablation}

\begin{figure}[htbp]
  \centering
  \begin{subfigure}[t]{0.49\textwidth}
    \centering
    \includegraphics[width=\textwidth]{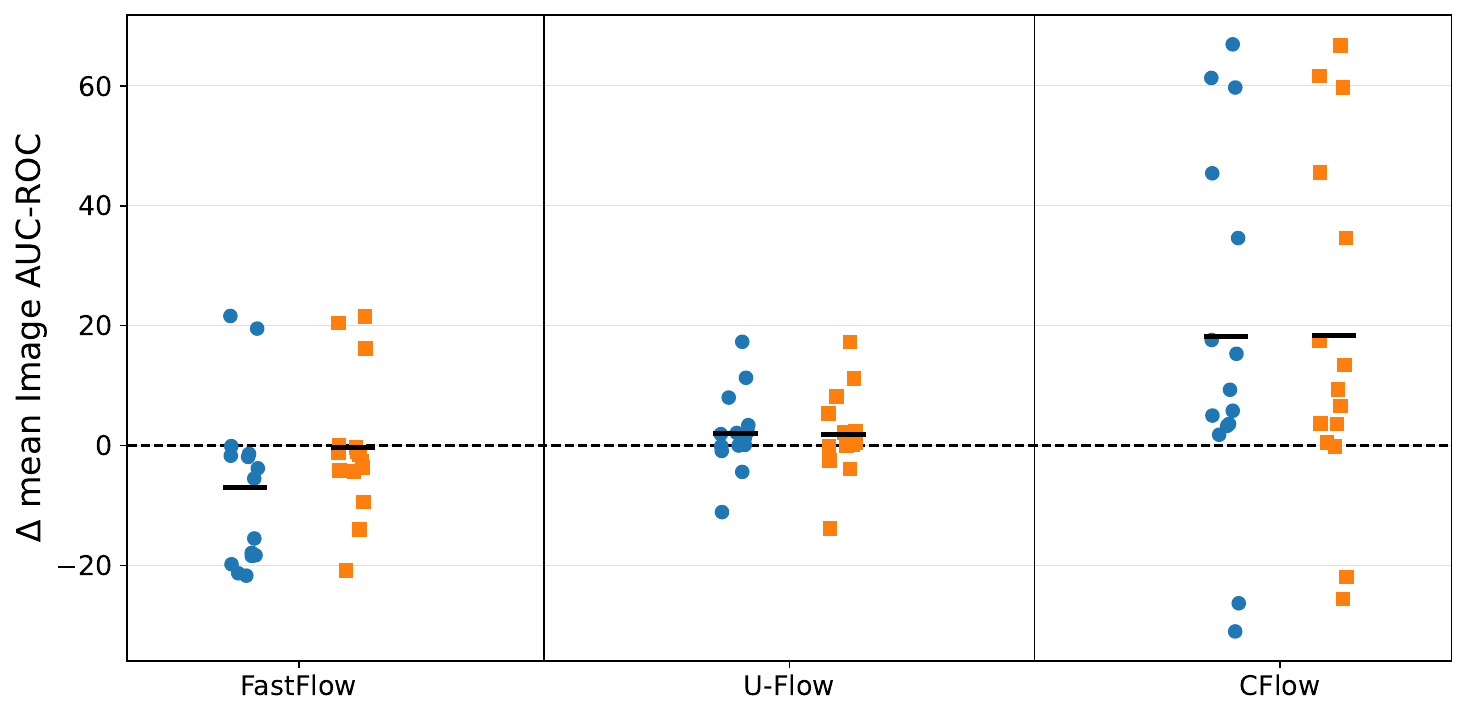}
    \subcaption{$\Delta$ mean image AUC-ROC (higher is better).}
    \label{fig:performance_ablation_distr}
  \end{subfigure}
  \hfill
  \begin{subfigure}[t]{0.49\textwidth}
    \centering
    \includegraphics[width=\textwidth]{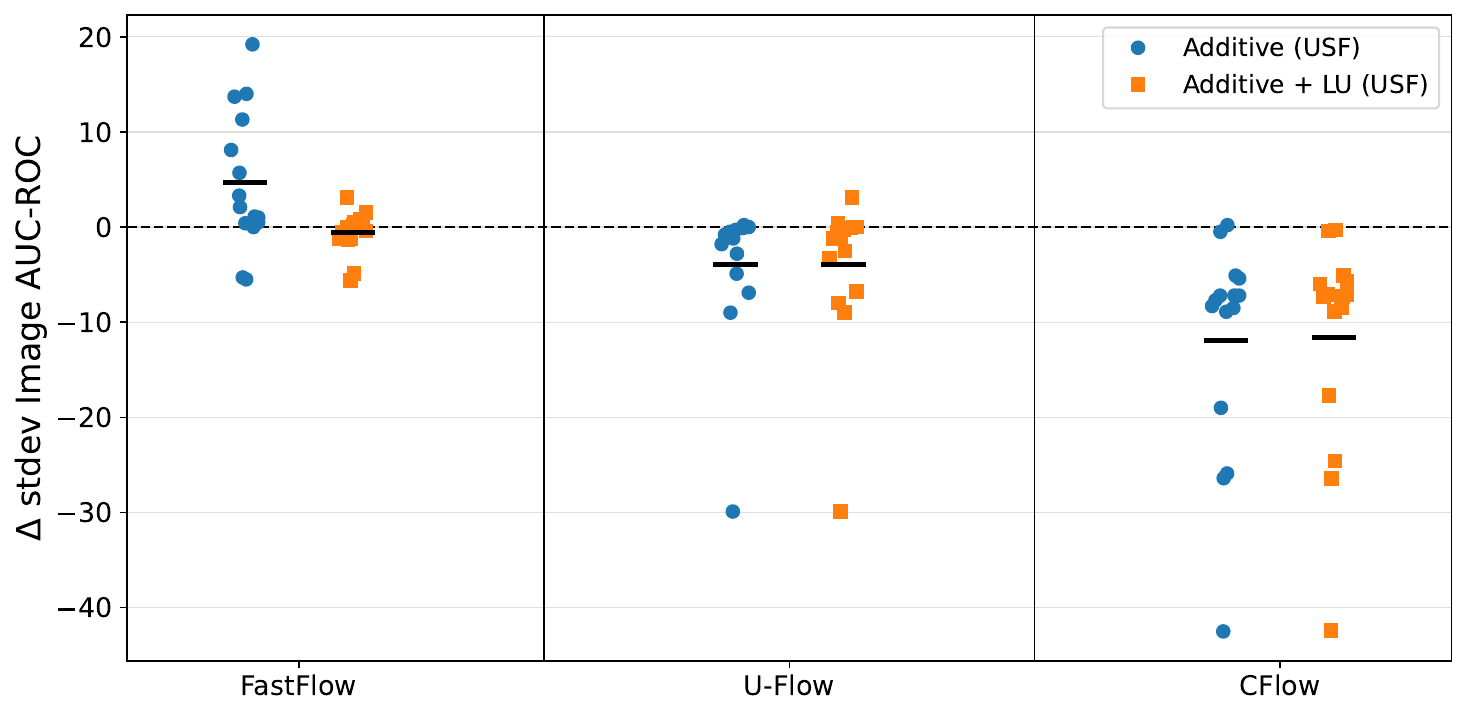}
    \subcaption{$\Delta$ stdev image AUC-ROC (lower is better).}
    \label{fig:performance_ablation_class}
  \end{subfigure}
    \begin{subfigure}[t]{0.49\textwidth}
    \centering
    \includegraphics[width=\textwidth]{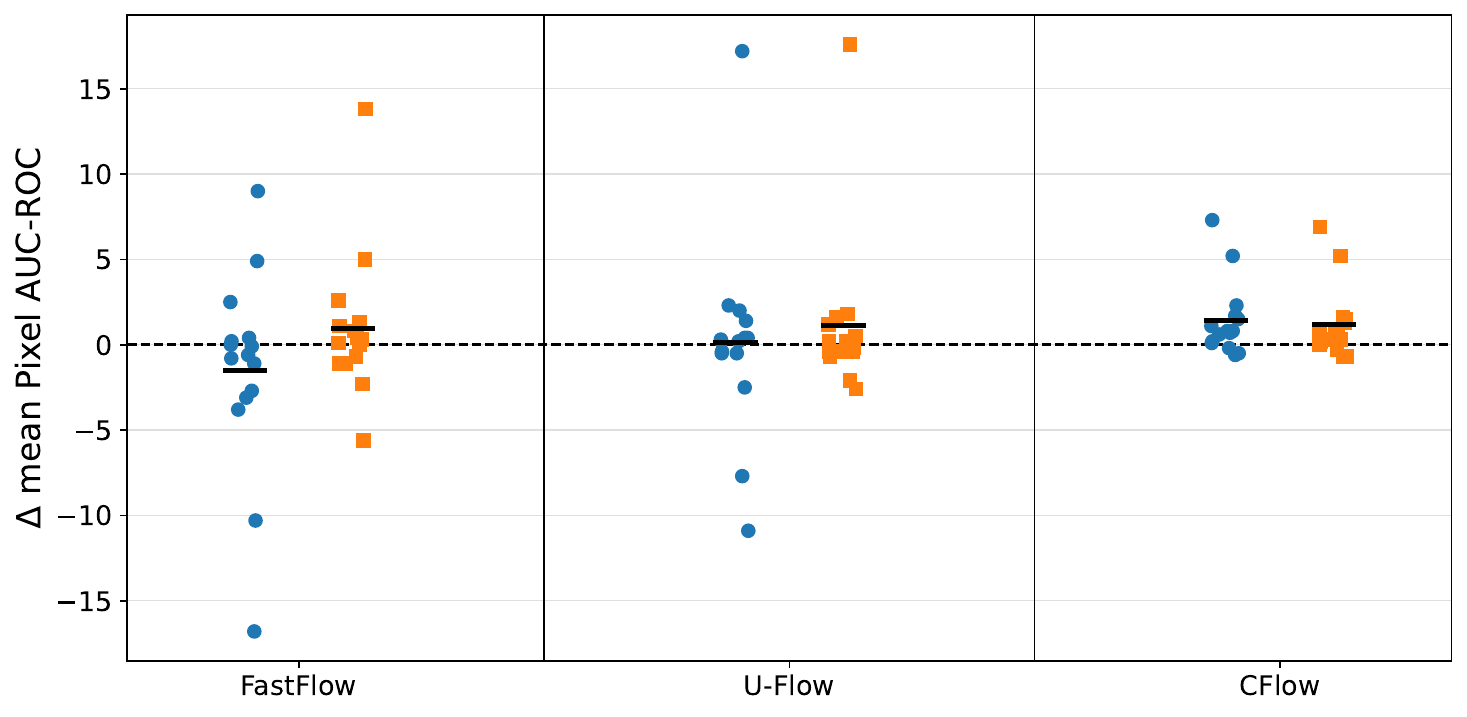}
    \subcaption{$\Delta$ mean pixel AUC-ROC (higher is better).}
    \label{fig:performance_ablation_distr}
  \end{subfigure}
  \hfill
  \begin{subfigure}[t]{0.49\textwidth}
    \centering
    \includegraphics[width=\textwidth]{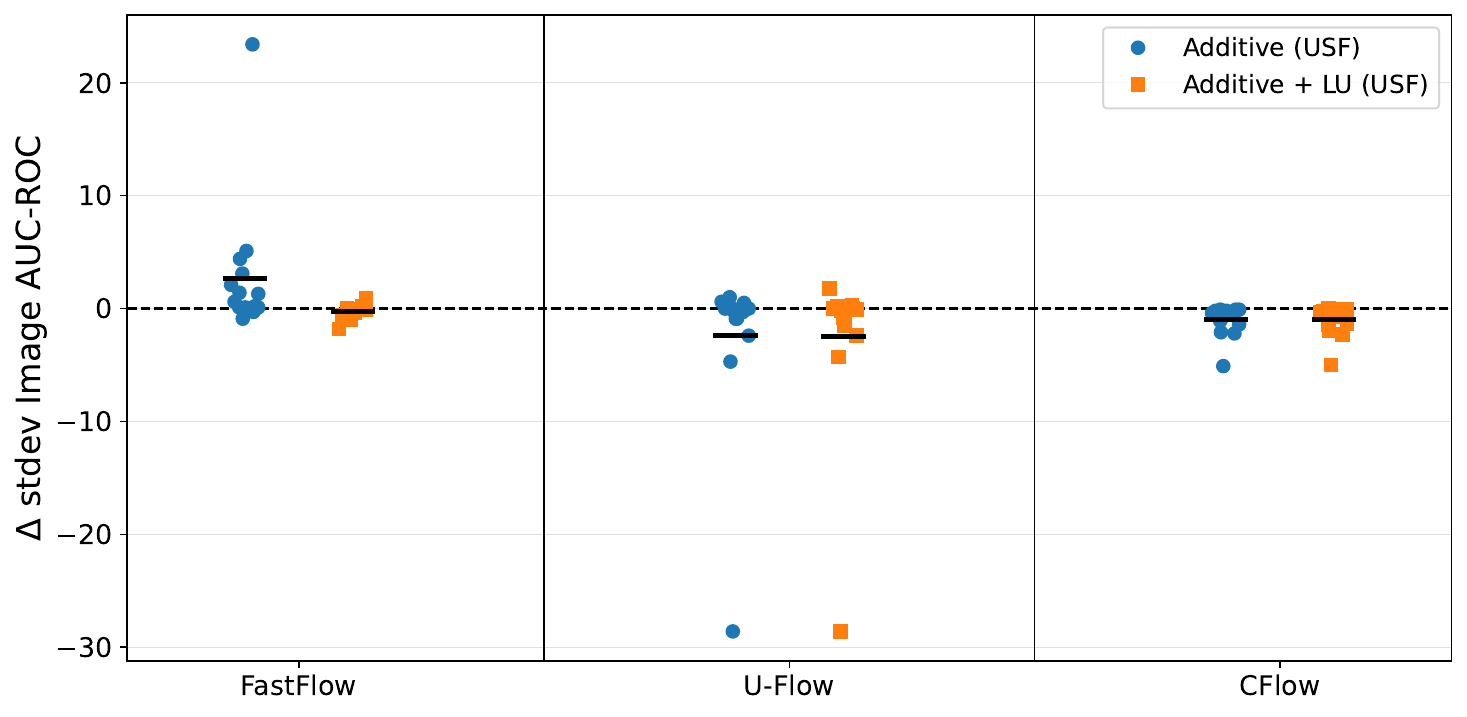}
    \subcaption{$\Delta$ stdev pixel AUC-ROC (lower is better).}
    \label{fig:performance_ablation_class}
  \end{subfigure}
  \caption{Ablation relative to the affine baseline per MVTec AD class. The dashed $y{=}0$ lines mark the affine baseline; the black horizontal mark is the class-wise mean $\Delta$. \emph{Additive (USF)} swaps affine for additive coupling; \emph{Additive + LU (USF)} additionally inserts LU-parameterized affine transforms. }
  \label{fig:performance_ablation}
\end{figure}

To disentangle which architectural change drives the stability and accuracy gains, we analyze changes against the affine baseline for two uniformly-scaling replacements: (1) an additive USF, which only swaps affine for additive coupling; (2) the full USF from the previous section, which introduces the LU-parameterized general bijective affine transforms as well. Figure~\ref{fig:performance_ablation} reports the deviations in mean and standard deviation per-class on MVTec AD for image-level and pixel-level AUC-ROC (\%).

On image-level, both ablations yield modest but consistent gains for U-Flow: the \emph{Additive (USF)} improves mean AUC-ROC by +2.09 while reducing the run-to-run standard deviation by --3.93, and the \emph{Additive + LU (USF)} achieves a similar improvement of +1.83 and --3.96 respectively. In contrast, CFlow exhibits large improvements under both variants: +18.15 mean $\Delta$ / --11.97 stdev. $\Delta$  (Additive) and +18.33 / --11.65 (Additive + LU), reflecting substantially higher accuracy and markedly reduced instability. For FastFlow, the full USF remains essentially on par with the affine baseline (--0.34 /  --0.60), while the USF only swapping affine for additive coupling performs strictly worse (--7.08 / +4.64). The pixel‑level ablation shows the same tendencies with generally smaller separations between USF versions and the non-USF baseline, with some positive outliers (e.g. the mean performance of FastFlow). Overall, these findings indicate that the additive coupling modification is the primary driver of improvements, while the additional LU-based bijective affine transforms prove to be a beneficial add-on, especially to preserve model expressivity in more lightweight architectures like FastFlow.

\section{Training instabilities of NonUSFlows}\label{sup:nonusf_instabilities}
During our experiments, we consistently observed numerical instabilities when training flows that employ affine coupling and LU-decomposed affine transforms on high dimensional data, which lead to failed training runs. An investigation revealed that locally exploding determinants cause the problem. The issue becomes more pronounced with larger networks. While the problem is known when working with affine coupling, it seems to get amplified by the LU layers. We performed limited parameter tuning for critical parameters such as the affine clamping (which cuts the multiplicative part of the output of the affine coupling layers to a predefined range) but we couldn't eliminate the instabilities completely. We assume that this combination requires very careful, potentially problem specific, tuning of hyper parameters such as the prior scale, the affine clamping and parameter initialization, which we couldn't perform during our experiments. In the following, we summarize instabilities that we encountered during our experiments. Note that we observed no such instabilities when working with the USFlows architecture nor when working with affine coupling layers and affine transforms with determinant one (cf. Section \ref{sec:experiments}), such as householder transforms or permutations (originally implemented in Anomalib). It is also noteworthy that the VAE-USFlow experiments and the ablation experiment do not share a common code basis since we implemented additive coupling and LU transforms in Anomalib independently.  

\paragraph{Gaussian Mixture Experiment} During hyperparameter optimization, we encountered two failed runs for NonUSFlows architecture, both for models using 10 coupling blocks. Nevertheless, the resulting model still proved to be very competitive w.r.t. the likelihood-likelihood comparison, which is why we accepted the results of the experiments. We did not observe any problems for the other dimensions.

\paragraph{VAE-NonUSFlows}
VAE-NonUSFlows shows very instable behavior and clearly suboptimal outcomes on MVTec. Figure \ref{fig:vae_results} documents failed runs and worse-than-random performance in detail.

\paragraph{Ablation Study}
We also encountered severe instabilities when using affine coupling and LU transforms in the ablation, in particular the different runs converged to the exact same (suboptimal) solution. Since we could not resolve the issue, we excluded the combination from the ablation. The obtained results on MVTec AD for the U-Flow architecture are reported in Figure~\ref{fig:ablation_full_line}.

\begin{figure}[htp!]
  \centering
    \begin{subfigure}[t]{\textwidth}
  \includegraphics[width=\textwidth]{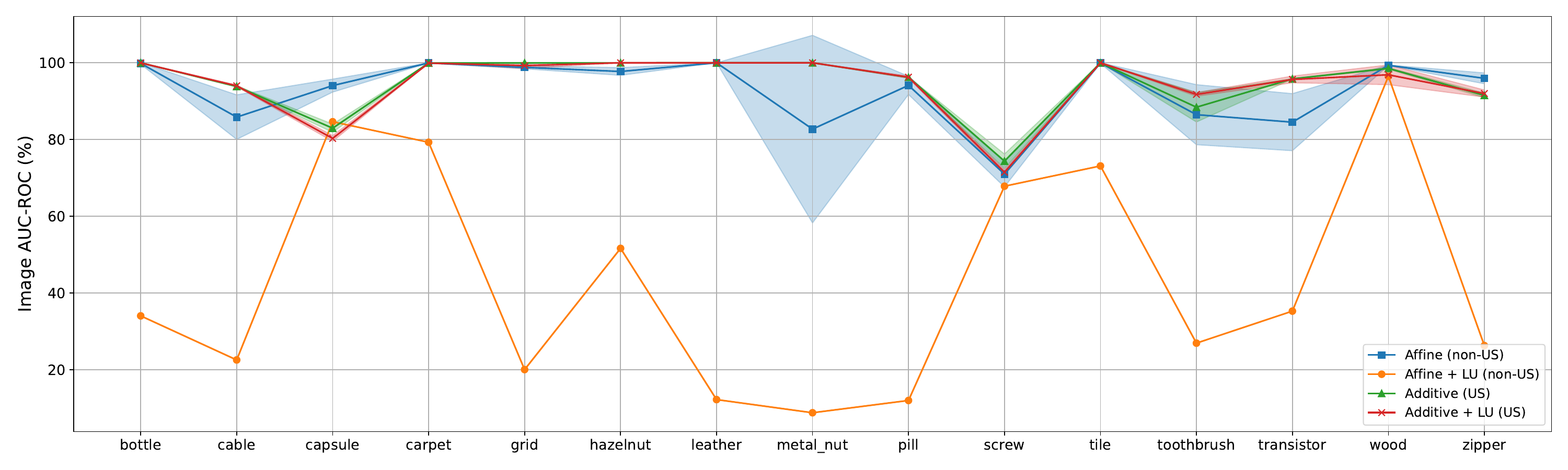}
  \caption{Image AUC-ROC}
  \label{fig:ablation_full_image}
  \end{subfigure}\hfill
  \begin{subfigure}[t]{\textwidth}
  \includegraphics[width=\textwidth]{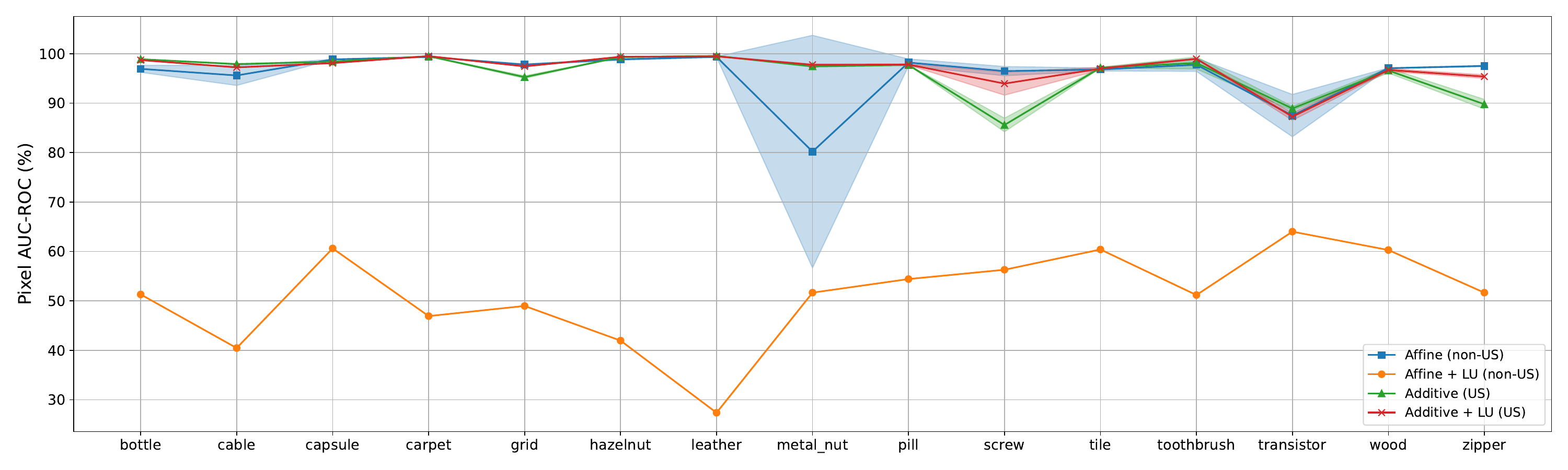}
  \caption{Pixel AUC-ROC}
  \label{fig:ablation_full_pixel}
  \end{subfigure}
  \caption{MVTec AD per-class mean (a) image- and (b) pixel-level AUC-ROC (\%) ± stdev. across 3 runs for the U-Flow architecture when using different flow layers.}
  \label{fig:ablation_full_line}
\end{figure}

\section{Conclusion}

This work establishes a formal equivalence between Deep SVDD and maximum-likelihood training of uniformly scaling flows (USFs), bridging two major paradigms in deep anomaly detection. This theoretical connection reveals that USFs inherently combine the distance-based reasoning of one-class methods with the density faithfulness of flow models, while their constant Jacobian determinant provides a natural regularization against representational collapse.
Empirically, we demonstrated that substituting standard flows with USFs in modern architectures (FastFlow, CFlow, U-Flow) acts as a simple yet effective drop-in replacement, consistently improving performance stability and often boosting accuracy across benchmarks.

From this perspective, several promising directions emerge. The role of the base distribution warrants deeper investigation, such as employing radial distributions defined by different norms (e.g., L1, Laplacian) or learning the norm distribution $P_{\|B\|}$ itself to better match the latent data structure. Furthermore, exploring more flexible base distributions, such as mixtures or heavy-tailed alternatives, could better capture complex, multi-modal normal data characteristics \cite{ZaidNY2024, Draxler_universality_2024}. Ultimately, this theoretical bridge provides a principled foundation for designing more robust, well-calibrated, and performant anomaly detectors.

\subsubsection*{Acknowledgments}
This research was in part funded by the Research Center Trustworthy Data Science and Security (https://rc-trust.ai), one of the Research Alliance centers within the UA Ruhr (https://uaruhr.de).

\bibliography{literature}
\bibliographystyle{plain}

\end{document}

%% file: math_commands.tex
\usepackage{amsmath,amsfonts,bm}

\def\eqref#1{equation~\ref{#1}}
\def\1{\bm{1}}

\DeclareMathAlphabet{\mathsfit}{\encodingdefault}{\sfdefault}{m}{sl}
\SetMathAlphabet{\mathsfit}{bold}{\encodingdefault}{\sfdefault}{bx}{n}